\documentclass[11pt,letterpaper]{article}

\usepackage[T1]{fontenc}
\usepackage[utf8]{inputenc}
\usepackage[english]{babel}
\usepackage[margin=1in]{geometry}
\usepackage{times}
\usepackage{amsmath,amssymb,mathtools,amsthm}
\usepackage{booktabs,graphicx,subcaption,microtype}
\usepackage{url,xurl,enumitem,float,multirow}
\usepackage{algorithm,algorithmic}
\usepackage[round,authoryear]{natbib}
\usepackage[hidelinks]{hyperref}

\graphicspath{{figures/}}
\DeclareMathOperator*{\argmin}{arg\,min}
\DeclareMathOperator{\Lip}{Lip}
\newcommand{\R}{\mathbb R}
\newcommand{\Torus}{\mathbb T}
\newcommand{\norm}[1]{\lVert #1\rVert}

\newcommand{\cmark}{$\checkmark$}
\newcommand{\xmark}{$\times$}
\theoremstyle{plain}
\newtheorem{theorem}{Theorem}
\newtheorem{corollary}{Corollary}[section]
\newtheorem{lemma}{Lemma}[section]

\theoremstyle{definition}
\newtheorem{assumption}{Assumption}

\title{Why Directly Learning Periodic Trajectories Can Fail}
\author{Kaixin Zheng \& Anita T. Layton \\
University of Waterloo \\
Waterloo, ON, Canada \\
\texttt{\{k59zheng,anita.layton\}@uwaterloo.ca}}
\date{}

\begin{document}
\maketitle
\begin{abstract}
Operator learning of periodic solutions requires deciding how simulation
data should be recorded and represented. A natural choice is to
integrate long enough for transients to decay and record a window wide
enough to contain at least one full period of all trajectories. We find that these conservative choices
can make the resulting trajectories difficult to learn, even when the
underlying periodic orbits vary regularly with system parameters.
Unaligned trajectories generalize poorly even within the training
distribution. Phase alignment substantially improves in-distribution generalization, but models trained on
a fixed physical-time window still have large errors on trajectories with
periods outside the training range.
We explain both failures through a common mechanism: frequency
differences accumulate over time, so the target phase varies rapidly
with the parameters. Predictors that cannot track this variation incur
a population MSE floor in both settings; for fixed window prediction, we also derive a per-sample lower
bound. We then study one of the simplest representations that escape
these floors: learning an aligned, normalized waveform and its period
separately. We establish
regularity of the decoupled targets under ODE assumptions and show
experimentally that this approach avoids both failures in ODE systems and a PDE case study. 
\end{abstract}

\section{Introduction}
Periodic dynamics are ubiquitous in science and engineering, from
biological rhythms and chemical oscillations to mechanical systems and
fluid flows \citep{vitaterna2001overview,cassani2021belousov,chai2021aeroelastic,juniper2018sensitivity}. A central problem across these applications is to
predict how the sustained periodic orbits change with system
parameters. 

Learning full trajectories and solution operators from system parameters is a
standard problem in scientific machine learning
\citep{karniadakis2021physics,kovachki2024operator,li2020fourier}. When the main interest is the long-time behavior rather than the transient
trajectory, it is natural to learn the long-time solution directly. For
systems converging to an equilibrium, one can learn the steady state
\citep{marwah2023deep}. Multiple steady states have also received some
attention \citep{zhang2023neural}. We study the analogous direct-learning
problem when the long-time solution is periodic rather than stationary.

Unlike a steady state, a periodic orbit does not determine a unique
trajectory target. The same orbit can be recorded from any starting phase,
and its period usually varies with the system parameters. A trajectory
target must therefore specify both a starting phase and a time
parameterization. Although these choices do not change the orbit itself,
they change the resulting parameter-to-trajectory map and may make accurate
prediction difficult.

\textbf{Related work.}
Our study is motivated by two practical issues in learning periodic
trajectories: phase misalignment and the use of wide, fixed time windows.
These difficulties have long been studied in other fields, where
phase alignment and time normalization are common remedies.
Phase ambiguity and time normalization are well known in numerical
methods for computing periodic orbits.
\citet{deuflhard1984computation} discusses the phase ambiguity and rescales time to the
unit interval in multiple shooting.
The boundary value formulation described by \citep[Chapter~10.1.4]{Kuznetsov2023}
uses the same time rescaling and imposes a phase condition to fix the
time origin. \citet{wilkins2009sensitivity} discuss how changes in the period produce
a time-growing term in parameter sensitivities and how rescaling time
removes this term.

Uncertainty quantification uses constant phase interpolation, stochastic
time warping, and rescaled time integration to approximate oscillatory
trajectories
\citep{witteveen2008alternative,mai2017surrogate,le2010asynchronous}.
In particular, \citet{mai2017surrogate} construct separate polynomial-chaos
surrogates for time-scaling and phase-shift parameters and the transformed
waveform, improving prediction over a direct time-based surrogate.
This provides a precedent for the period--waveform decomposition
considered here.
Functional data analysis uses curve registration and phase-and-amplitude
separation to distinguish timing variation from waveform variation
\citep{marron2015functional,tucker2013generative,ramsay1998curve}.
Time series forecasting uses phase alignment across periodic windows
\citep{wang2026aligntime}, while other data-driven methods learn asymptotic
phase, coupling functions, or reduced phase dynamics
\citep{takata2023definition,hwang2026data,wilson2024data}.

\textbf{Contributions.}
We study a general family of coordinate choices that can cause error
floors when learning periodic trajectories. Phase misalignment and
wide fixed-window recording are two cases within this family.
Our contributions are:

\begin{itemize}
\item
We mathematically and quantitatively prove that this family of phase coordinates produces a population
error floor under the stated regularity assumptions.
For fixed-window targets, we also derive a lower bound on per-sample
out-of-distribution (OOD) error.

\item
We design a series of ablation studies to isolate and verify each
failure mode in multiple dynamical systems. We then perform parameter sweeps for each failure mode
with multiple neural architectures to test the predicted error floors.
The results agree with the theory across the tested coordinate choices,
failure modes, dynamical systems, and neural architectures.

\item
We use the theorem to guide target design, with the established
period--waveform decomposition as one example. Beyond avoiding the
identified failure mechanism, we prove that this target is
neural-network approximable under the stated dynamical-system assumptions.
We also demonstrate physical-time reconstruction from the learned maps.
\end{itemize}
\section{When Conservative Data Preparation Hurts}
\label{sec:failure_modes}
\textbf{Motivation.} Suppose we can generate training trajectories, either by numerically
solving an ODE or PDE or by collecting data from a physical system whose
parameters we can control. We must still decide when to begin recording
and how long to record. 

One natural choice is to wait until the same,
conservatively late time for every parameter value, so that the
trajectories are expected to have reached their periodic orbits. Another
is to record over the same, conservatively wide physical-time window
$[0,W]$, chosen to cover the longer periods we expect, including at
parameter values not seen during training. Counterintuitively, we find
that each choice can independently create a distinct generalization
failure when learning periodic trajectories.

\textbf{Formalization.} We formalize this setting as a parameterized autonomous system:
\begin{equation}
\dot x(t)=f(x(t),p),\qquad
x(0)=x_0,\qquad p\in\mathcal P\subset\R^m.
\label{eq:parametric_ode}
\end{equation}
We assume that, given parameter \(p\), each trajectory approaches a unique attracting periodic orbit
$\mathcal C(p)$ with the period $T(p)$. Our goal is to learn how the
resulting cycle varies with $p$, rather than the full trajectory from
$x_0$, which also includes the initial approach to the cycle. We
examine the two recording choices in the following ablations.

\subsection{Two choices, two failure modes}
\label{sec:problem_setting}

To examine the two recording choices separately, we use the Goodwin
oscillator \citep{gonze2021goodwin} and the higher-dimensional Goldbeter
circadian model (\textsc{gold95}) \citep{goldbeter1995model} as ordinary differential equation (ODE) testbeds.
Both have stable periodic orbits in the selected parameter domains. We
compare waveform targets with and without phase alignment and period
normalization under IID and period-based out-of-distribution (OOD) splits. Dataset and model
details are given in Appendices~\ref{app:odes} and~\ref{app:ablation_setup}.
All reported waveform MSE values are computed after applying a
Yeo--Johnson transform and standardizing each output variable using
the training data.

\noindent\textbf{Failure mode 1.}
We first establish a working baseline. Motivated by representations
that separate timing from shape in other fields
\citep{marron2015functional,tucker2013generative,ramsay1998curve}, we
learn the period and waveform separately. For each trajectory, we
extract one cycle, map it to a common phase interval, and align its
waveform by matching peaks of a reference variable. The resulting
waveforms are learned accurately: mean IID validation MSE is
$3.79\times10^{-4}$ on Goodwin and $7.19\times10^{-4}$ on
\textsc{gold95}. We discuss how the learned period and waveform can be
used for trajectory prediction in Section~\ref{sec:snapshot}.

Does this success require both preprocessing steps? We first remove
period normalization while retaining phase alignment. Somewhat
surprisingly, the fixed-window targets remain easy to learn under IID
splits: mean validation MSE is $5.48\times10^{-5}$ on Goodwin and
$8.28\times10^{-4}$ on \textsc{gold95}. On Goodwin, the error is even
lower than with period normalization.

Removing phase alignment produces a different result. On Goodwin, we
integrate from a shared initial condition for a long fixed time and
record without aligning the resulting waveforms. Mean IID validation
MSE rises to $0.997$ without period normalization and $0.970$ with it,
although training error remains below $1.1\times10^{-3}$ in both
cases. This test includes both the initial approach to the cycle and
subsequent movement along it. To isolate the latter, we align the
converged orbits of \textsc{gold95} first, then evolve each for the
same long time. Its unaligned targets again have mean IID validation
MSE above $1$. A common long integration after alignment is therefore
sufficient to produce this failure, which we call \emph{FM1}.

\noindent\textbf{Failure mode 2.}
The IID ablations show that phase-aligned fixed-window targets can
achieve low validation error without period normalization. But will
this still hold when validation periods extend beyond the training
range? Intuitively, it's possible because partial trajectory is still a trajectory.

Surprisingly, the ablation study tells a different story. We train on shorter-period trajectories and validate on longer-period
ones, setting $W$ to the longest training period. Fixed-window targets
have mean OOD validation MSE of $0.228$ on Goodwin and $0.200$ on
\textsc{gold95}; period normalization reduces these errors to $0.0104$
and $0.0163$, respectively. Thus, under the OOD split, lack of period normalization increases validation error by more than an order of magnitude. We call this failure mode 2 (\emph{FM2}).

Table~\ref{tab:ablation_both} reports the ablations over five training
seeds. Both failure modes occur consistently across seeds. For FM1,
unaligned IID targets have validation MSE near or above $1$, whereas
their phase-aligned counterparts remain below $10^{-3}$. For FM2, fixed-window OOD validation MSE is more than 10-fold higher than the period-normalized MSE on both systems. The standard
deviations and ranges across seeds are small relative to these gaps,
and the MSE ranges of the compared settings do not overlap.

\begin{table*}[htpb]
\caption{ODE waveform ablations over five training seeds.
Checkmarks indicate period normalization and phase alignment. Experiments under 5 training seeds are included.  SD: standard deviation.}
\label{tab:ablation_both}
\centering
\footnotesize
\setlength{\tabcolsep}{3pt}
\renewcommand{\arraystretch}{0.95}
\begin{tabular}{@{}cccccc@{}}
\toprule
\multicolumn{6}{l}{\textbf{Goodwin}} \\
\midrule
Norm. & Align. & Train mean (SD) & Val mean (SD) & Val min & Val max \\
\midrule
\multicolumn{6}{l}{\textit{IID split}} \\
$\checkmark$ & $\checkmark$
& $3.36\!\times\!10^{-4}\;(9.98\!\times\!10^{-7})$
& $3.79\!\times\!10^{-4}\;(3.54\!\times\!10^{-6})$
& $3.74\!\times\!10^{-4}$ & $3.83\!\times\!10^{-4}$ \\
$\times$ & $\checkmark$
& $3.03\!\times\!10^{-5}\;(1.55\!\times\!10^{-6})$
& $5.48\!\times\!10^{-5}\;(4.59\!\times\!10^{-6})$
& $4.94\!\times\!10^{-5}$ & $6.17\!\times\!10^{-5}$ \\
$\times$ & $\times$
& $7.74\!\times\!10^{-4}\;(8.74\!\times\!10^{-5})$
& $9.97\!\times\!10^{-1}\;(1.96\!\times\!10^{-2})$
& $9.76\!\times\!10^{-1}$ & $1.03$ \\
$\checkmark$ & $\times$
& $1.03\!\times\!10^{-3}\;(4.40\!\times\!10^{-5})$
& $9.70\!\times\!10^{-1}\;(1.56\!\times\!10^{-2})$
& $9.53\!\times\!10^{-1}$ & $9.91\!\times\!10^{-1}$ \\
\midrule
\multicolumn{6}{l}{\textit{OOD split}} \\
$\checkmark$ & $\checkmark$
& $2.99\!\times\!10^{-4}\;(6.07\!\times\!10^{-7})$
& $1.04\!\times\!10^{-2}\;(1.89\!\times\!10^{-3})$
& $8.14\!\times\!10^{-3}$ & $1.25\!\times\!10^{-2}$ \\
$\times$ & $\checkmark$
& $2.84\!\times\!10^{-4}\;(2.90\!\times\!10^{-6})$
& $2.28\!\times\!10^{-1}\;(2.22\!\times\!10^{-2})$
& $1.94\!\times\!10^{-1}$ & $2.54\!\times\!10^{-1}$ \\
\midrule
\multicolumn{6}{l}{\textbf{\textsc{gold95}}} \\
\midrule
Norm. & Align. & Train mean (SD) & Val mean (SD) & Val min & Val max \\
\midrule
\multicolumn{6}{l}{\textit{IID split}} \\
$\checkmark$ & $\checkmark$
& $4.27\!\times\!10^{-4}\;(7.72\!\times\!10^{-6})$
& $7.19\!\times\!10^{-4}\;(9.07\!\times\!10^{-6})$
& $7.10\!\times\!10^{-4}$ & $7.32\!\times\!10^{-4}$ \\
$\times$ & $\checkmark$
& $3.08\!\times\!10^{-4}\;(6.48\!\times\!10^{-6})$
& $8.28\!\times\!10^{-4}\;(3.29\!\times\!10^{-5})$
& $7.99\!\times\!10^{-4}$ & $8.80\!\times\!10^{-4}$ \\
$\times$ & $\times$
& $9.75\!\times\!10^{-4}\;(1.12\!\times\!10^{-4})$
& $1.328\;(3.23\!\times\!10^{-2})$
& $1.303$ & $1.378$ \\
$\checkmark$ & $\times$
& $1.48\!\times\!10^{-3}\;(3.68\!\times\!10^{-5})$
& $1.192\;(1.18\!\times\!10^{-2})$
& $1.175$ & $1.209$ \\
\midrule
\multicolumn{6}{l}{\textit{OOD split}} \\
$\checkmark$ & $\checkmark$
& $3.82\!\times\!10^{-4}\;(7.62\!\times\!10^{-6})$
& $1.63\!\times\!10^{-2}\;(3.63\!\times\!10^{-3})$
& $1.34\!\times\!10^{-2}$ & $2.21\!\times\!10^{-2}$ \\
$\times$ & $\checkmark$
& $3.19\!\times\!10^{-4}\;(1.46\!\times\!10^{-5})$
& $2.00\!\times\!10^{-1}\;(1.22\!\times\!10^{-2})$
& $1.85\!\times\!10^{-1}$ & $2.19\!\times\!10^{-1}$ \\
\bottomrule
\end{tabular}
\end{table*}

\section{Error Floors for Parameterized Periodic Targets}
\label{sec:theory}

We study a family of phase-coordinate choices that includes FM1 and
FM2 as two examples. The underlying difficulty is intuitive: a large
time scale can turn small frequency differences between nearby
parameter values into substantial phase differences. A predictor
that changes slowly with the parameters then gives similar
predictions for trajectories that can be far out of phase.

\subsection{Setup and problem formulation}

\textbf{Common form.}
For each parameter $p$, let $U(\cdot;p)$ denote a phase-aligned
representation of the periodic orbit, with frequency
$\omega(p)=1/T(p)$. The phase reference is fixed throughout the
analysis, and $U(\cdot;p)$ is one-periodic on
$\Torus=\R/\mathbb Z$, so that $x_p(t)=U(t\omega(p);p)$.
We consider targets of the form
\begin{equation}
\Phi_\lambda(s,p)
=
\alpha(s)+\lambda\beta(s)\omega(p),
\qquad
F_\lambda(p)(s)
=
U\bigl(\Phi_\lambda(s,p);p\bigr),
\label{eq:coordinate_target}
\end{equation}
where $s\in[0,1]$, $\alpha$ and $\beta$ are fixed functions, and
$\lambda$ controls the time scale.
FM1 corresponds to $(\alpha,\beta,\lambda)=(s,1,c)$, giving
$F_c^{(1)}(p)(s)=U(s+c\omega(p);p)$.
FM2 corresponds to $(\alpha,\beta,\lambda)=(0,s,W)$, giving
$F_W^{(2)}(p)(s)=U(Ws\omega(p);p)$ after rescaling physical time
by $t=Ws$.
For the general family,
$\nabla_p\Phi_\lambda(s,p)=\lambda\beta(s)\nabla\omega(p)$.

\textbf{Prediction problem.}
Let $X=\R^n$ for an ODE or $X=L^2(\mathcal D_x;\R^n)$ for a PDE.
The targets and predictions belong to
$\mathcal H=L^2([0,1];X)$, with
$\norm{v}_{\mathcal H}^2=\int_0^1\norm{v(s)}_X^2\,ds$.
Given a parameter distribution $\mu$ on $\mathcal P$, we study
the population error of a trajectory predictor
$g_\lambda:\mathcal P\to\mathcal H$,
$\mathrm{MSE}_\lambda=\norm{g_\lambda-F_\lambda}_{\mu,2}^2$,
where
$\norm{v}_{\mu,2}^2
=\int_{\mathcal P}\norm{v(p)}_{\mathcal H}^2\,d\mu(p)$.
Our bounds relate this error to how rapidly the predictor and
the target vary with the parameter.

\textbf{Phase mean and variance.}
For each $p$, define the \emph{phase mean}
$\bar U(p)=\int_0^1U(\tau;p)\,d\tau$, also viewed as a constant
trajectory in $\mathcal H$, and the \emph{orbit variance}
$V(p)=\int_0^1\norm{U(\tau;p)-\bar U(p)}_X^2\,d\tau$.
The orbit variance measures variation along one periodic orbit.
Averaging it over the parameter distribution gives the
\emph{population phase variance}
$\sigma_{\mathrm{phase}}^2=\int_{\mathcal P}V(p)\,d\mu(p)$.

\begin{assumption}[Target and parameter distribution regularity]
\label{ass:regularity}
\emph{(i)} $\mathcal P\subset\R^m$ is compact, has nonempty interior,
and has Lebesgue-null boundary.
\emph{(ii)} $\omega$ has a positive $C^1$ extension to an open
neighborhood of $\mathcal P$.
\emph{(iii)} $U:\Torus\times\mathcal P\to X$ is continuous and satisfies
$\sup_{\tau\in\Torus}\norm{U(\tau;p)-U(\tau;q)}_X
\le K_U\norm{p-q}_2$ for some $K_U<\infty$.
\emph{(iv)} $\mu$ has a continuous density
$d\mu(p)=\varpi(p)\,dp$ on $\mathcal P$.
\emph{(v)} $\mu(\{p:\nabla\omega(p)=0\})=0$.
\end{assumption}

\subsection{Population error}
\label{sec:phase_averaging}
\label{sec:fm1}
\label{sec:fm2}

For large $\lambda$, nearby parameters produce widely varying phases
but nearly identical predictions. Averaging over these phases favors
predicting the phase mean, leaving the phase variance
$\sigma_{\mathrm{phase}}^2$ as the error floor.

\begin{theorem}[Population error floor and decomposition]
\label{thm:unified}
Under Assumption~\ref{ass:regularity}, suppose
$\beta(s)\ne0$ almost everywhere. Let the trajectory predictor $g_\lambda:\mathcal P\to\mathcal H$
be Lipschitz with respect to the parameter $p$, with constant
$L_\lambda=o(\lambda)$. Set
$\mathrm{MSE}_\lambda
=\norm{g_\lambda-F_\lambda}_{\mu,2}^2$, where we define
$\norm{v}_{\mu,2}^2
=\int_{\mathcal P}\norm{v(p)}_{\mathcal H}^2\,d\mu(p)$. Then
\begin{equation}
\liminf_{\lambda\to\infty}\mathrm{MSE}_\lambda
\ge \sigma_{\mathrm{phase}}^2.
\label{eq:unified_floor}
\end{equation}
If $\mathrm{MSE}_\lambda$ remains bounded, then
\begin{equation}
\mathrm{MSE}_\lambda
=
\sigma_{\mathrm{phase}}^2
+
\norm{g_\lambda-\bar U}_{\mu,2}^2
+
o(1).
\label{eq:unified_decomposition}
\end{equation}
\end{theorem}

\noindent\textit{Proof sketch.}
Use frequency as a local parameter coordinate and divide it into cells on
which the predictor and fixed-phase waveform vary little, while the target
completes many phase turns. Phase averaging removes the correlation
between the predictor and the centered waveform $U-\bar U$. Expanding
the squared error then gives \eqref{eq:unified_decomposition}. Full
details are given in Appendix~\ref{app:unified}.

\textbf{Interpretation.}
The decomposition shows that attaining the asymptotic error floor
requires $\norm{g_\lambda-\bar U}_{\mu,2}\to0$.
Thus, predictions approaching the floor converge to the phase mean,
which is constant in phase but can vary with $p$.

For a fixed feedforward architecture, the predictor's Lipschitz
constant is bounded by the product of the layerwise Lipschitz
constants \citep{gouk2021regularisation}. Uniform bounds on these
constants as $\lambda$ grows therefore imply $L_\lambda=O(1)$,
which is sufficient for $L_\lambda=o(\lambda)$.

\textbf{Applications.}
Applying Theorem~\ref{thm:unified} to FM1 and FM2 gives
\begin{equation}
\begin{aligned}
\liminf_{c\to\infty}
\norm{g_c^{(1)}-F_c^{(1)}}_{\mu,2}^2
\ge \sigma_{\mathrm{phase}}^2,\quad
\liminf_{W\to\infty}
\norm{g_W^{(2)}-F_W^{(2)}}_{\mu_{\mathrm{OOD}},2}^2
\ge \sigma_{\mathrm{phase}}^2.
\end{aligned}
\label{eq:fm1_fm2_floors}
\end{equation} 
Here, for second inequality, $\mu_{\mathrm{OOD}}$ is a parameter distribution
independent of $W$, with frequencies outside the training range, and
$\sigma_{\mathrm{phase}}^2
=\int_{\mathcal P}V(p)\,d\mu_{\mathrm{OOD}}(p)$.
Section~\ref{sec:numerics} examines
the theorem's applicability through empirical estimates of the
predictors' Lipschitz constants and provides numerical validation of error floors.

\subsection{Per-sample OOD error}
\label{sec:pointwise}

The FM2 experiments show that OOD prediction can fail. This leaves us
curious about two questions: How far must a new sample's frequency lie
outside the training range before its error becomes substantial?
Which quantities determine the size of that error?

To quantitatively derive the bound, let $\Omega_{\mathrm{train}}=[\omega_-,\omega_+]$ be the training
frequency range, and fix one endpoint $\omega_b$. Define
\[
\Gamma=\{p\in\mathcal P:\omega(p)=\omega_b\},
\qquad
d_\omega(p)=|\omega(p)-\omega_b|,
\qquad
e_W(p)=\norm{g_W^{(2)}(p)-F_W^{(2)}(p)}_{\mathcal H}.
\]

\begin{assumption}[FM2 boundary regularity]
\label{ass:fm2_boundary}
There exists $r_0\in\Gamma\cap\operatorname{int}\mathcal P$
such that $\nabla\omega(r_0)\ne0$ and $V(r_0)>0$.
Moreover, there exist an open neighborhood
$G\subset\operatorname{int}\mathcal P$ of $r_0$
and a constant $K_\tau<\infty$ such that
\(
\norm{U(\tau+v;p)-U(\tau;p)}_X\le K_\tau|v|
\)
for all $p\in G$, $\tau\in\Torus$, and $v\in\R$.
\end{assumption}

\begin{theorem}[FM2: per-sample OOD lower bound]
\label{thm:fm2_pointwise}
Under Assumption~\ref{ass:regularity}(i)--(iii) and
Assumption~\ref{ass:fm2_boundary}, suppose $g_W^{(2)}$ has
parameter Lipschitz constant $L_W=o(W)$ and $e_W\to0$
uniformly on a fixed neighborhood of $r_0$ in $\Gamma$.
Then there exists a neighborhood $\mathcal N\subset G$
of $r_0$ such that, for every fixed $B>0$,
\begin{equation}
\frac{e_W(p)^2}{V(p)}
\ge \min\{z(p)^2,1\}-o(1),
\qquad z(p)=Wd_\omega(p),
\label{eq:pointwise_uniform}
\end{equation}
uniformly over OOD points $p\in\mathcal N$ satisfying
$0<z(p)\le B$, as $W\to\infty$.
\end{theorem}

\noindent\textit{Proof sketch.}
Compare each OOD prediction with its value at a nearby training
boundary point, and freeze the waveform at the test parameter.
Periodic averaging bounds the resulting trajectory discrepancy
below by $V(p)\min\{z(p)^2,1\}-o(1)$.
Appendix~\ref{app:pointwise} gives the full proof.

\textbf{Interpretation.} Surprisingly, when the observation window is
wide, a sample just outside the training range can already be predicted
poorly, even if a well-trained predictor has small error at the boundary frequency. Here is a straightforward explanation: Changing slowly with the parameters, a predictor gives nearly the same
waveforms for a parameter on the training boundary and for an OOD parameter
close to it. But the two true trajectories have slightly different frequencies. Far
enough into a wide window, their phase difference keeps changing at
every time step and can reach up to a full cycle, which makes the predictor
fail. Numerical
validation is given in Section~\ref{sec:numerics}.

\textbf{Applicability.} For Goodwin and \textsc{gold95}, the phase Lipschitz condition holds on
a compact parameter region because
$\partial_\tau U(\tau;p)=T(p)f(U(\tau;p);p)$, where both the period and governing equation, \(f\), in \eqref{eq:parametric_ode} are continuous. For the point $r_0$, it
suffices that varying a single parameter, with the others fixed,
reaches $\omega_b$ inside the box with a nonzero partial derivative. We
verified this numerically for both systems. Experimental details in Section~\ref{app: verify_fm2_two_system}.

\section{Numerical validation}\label{sec:numerics}

Theorem~\ref{thm:unified} predicts that the population MSE has a floor once the burn-in shift $c$ or the window length $W$ is large enough. Theorem~\ref{thm:fm2_pointwise} predicts that, for test points close enough to the boundary of the training set, the error on individual samples remains large. We test both predictions numerically.

\subsection{FM1 and FM2 applicability}

Before testing the predicted error floors, we check empirically whether
the predictors' Lipschitz constants satisfy $L_\lambda=o(\lambda)$.
We estimate local Lipschitz constants from the largest singular values
of the predictors' Jacobians at sampled validation points. We then
track the maximum estimate divided by $\lambda$, where $\lambda=c$
in FM1 and $\lambda=W$ in FM2.
Appendix~\ref{app:lip_check} provides the estimation procedure and
full sweep plots.

In FM1, the ratio drops from $1.17$ to $0.0115$ for Goodwin as $c$
increases from $10^2$ to $10^4$, and from $1.77$ to $0.00322$ for
\textsc{gold95} as $c$ increases from $10^2$ to $10^5$. In FM2, the
ratio on the OOD validation set decreases more gradually. Over a
$64$-fold increase in $W$, it falls from $1.81$ to $0.38$ for Goodwin
and from $1.52$ to $0.23$ for \textsc{gold95}. The sharper decline in
FM1 may help explain why its validation errors come closer to the
predicted error floor in the experiments below than that in FM2.

\subsection{FM1: shift sweep}

We train each shift configuration with five seeds and report the
lowest validation error reached during training (\emph{Best val MSE})
and the final validation error (\emph{Final val MSE}).
Figure~\ref{fig:fm1_c_sweep} shows that both are well below
$\sigma^2_{\mathrm{phase}}$ at small shifts, but reach or exceed it
at larger tested shifts on both systems. Training configurations
and detailed results are given in Appendices~\ref{app:fm1_config}
and~\ref{app:fm1_sweep_results}.

At large shifts, Best val MSE is reached within the first four
epochs and closely matches $\sigma^2_{\mathrm{phase}}$: its mean
across seeds exceeds $\sigma^2_{\mathrm{phase}}$ by at most 0.6\%, with standard deviation at most 0.18\% of $\sigma^2_{\mathrm{phase}}$. By our error decomposition, a
near-floor error means that the predictor is close to the orbit's
phase average, with little variation in phase. After the early validation minimum, training MSE continues to
decrease while validation MSE rises. Final val MSE reaches
1.69--1.82 times the floor, showing clear overfitting during
continued training.

\begin{figure}[htbp]
    \centering
    \includegraphics[width=0.7\linewidth]{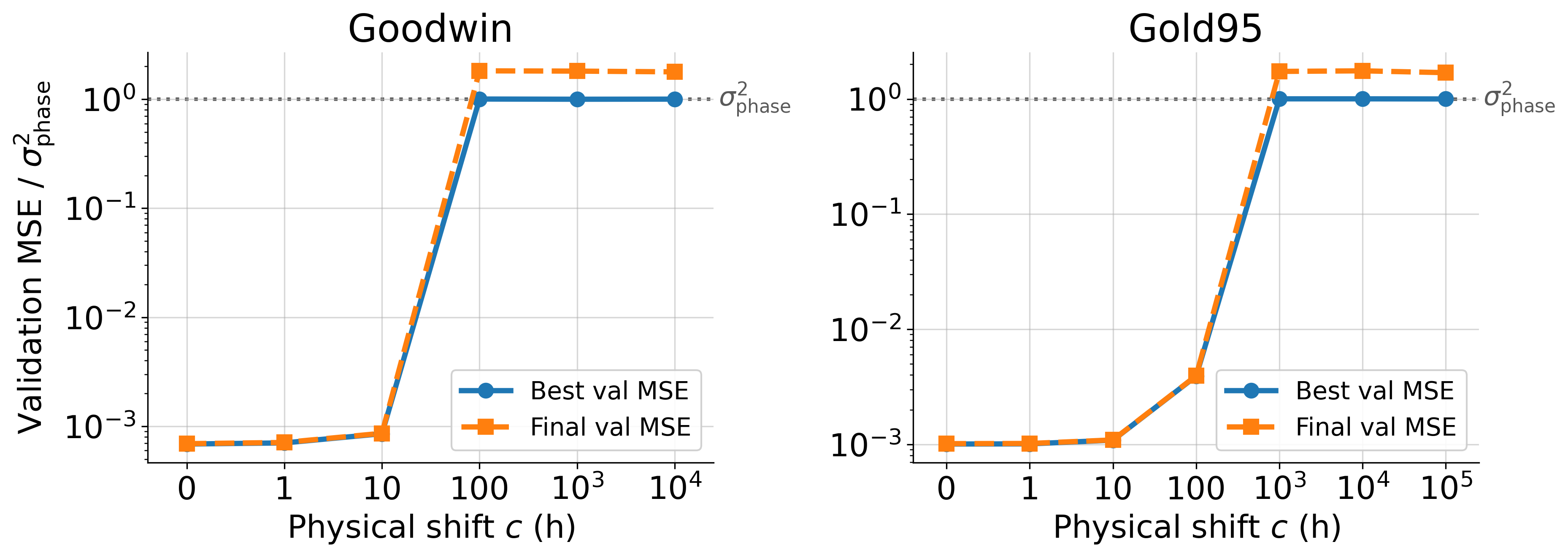}
    \caption{FM1 shift sweep. Validation MSE divided by $\sigma^2_{\mathrm{phase}}$ against the burn-in shift $c$. Markers show the mean over five seeds. The range across seeds is below 8\% of the mean at every point except Gold95 at $c = 100$\,h (14\%), which lies in the transition. Dotted line: $\sigma^2_{\mathrm{phase}}$.}
    \label{fig:fm1_c_sweep}
\end{figure}

\subsection{FM2: window sweep}
\label{sec:fm2-validation}

To test the predictions of Theorems~\ref{thm:unified}
and~\ref{thm:fm2_pointwise} as the window length $W$ grows, we sweep $W$ on
Goodwin and GOLD95 and train five seeds per $W$. Configuration details and
per-$W$ results are in Appendices~\ref{app:FM2_config}
and~\ref{app:FM2_results}.

\begin{figure}[htbp]
    \centering
    \includegraphics[width=0.7\linewidth]{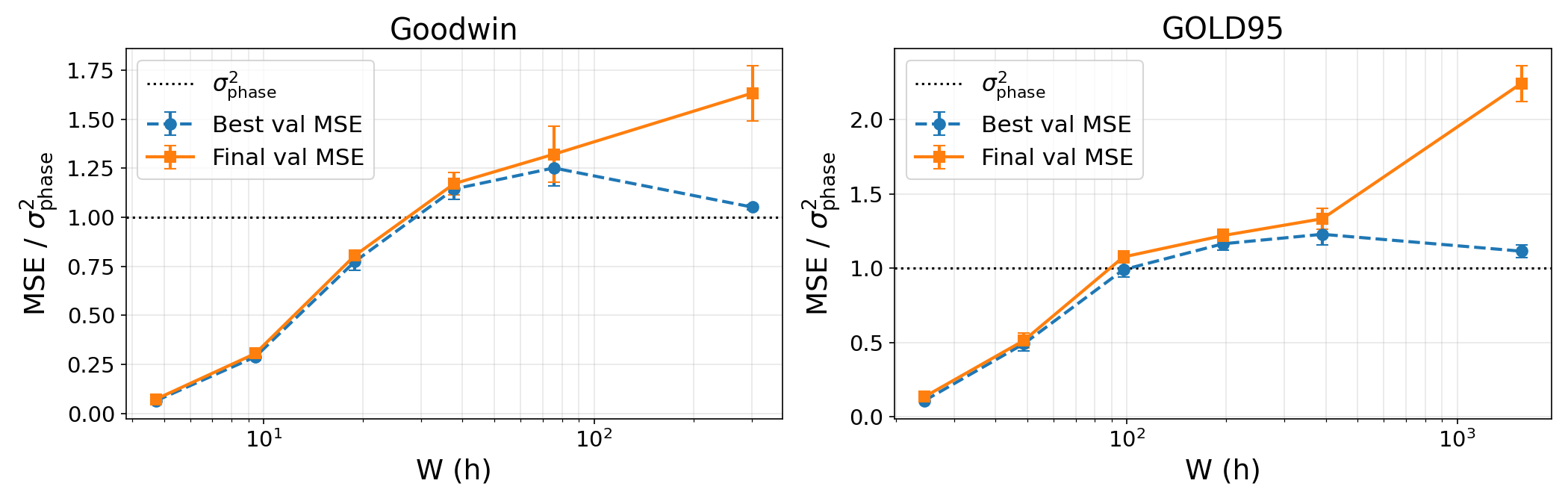}
    \caption{Population validation MSE normalized by $\sigma^2_{\mathrm{phase}}$ across the $W$ sweep, for Goodwin (left) and GOLD95 (right). Error bars are one standard deviation over 5 seeds.}
    \label{fig:fm2-pop}
\end{figure}

\textbf{Population error behavior.} Theorem~\ref{thm:unified} predicts that, given a huge \(W\), the
population error stays above $\sigma^2_{\mathrm{phase}}$ for a predictor
with limited parameter sensitivity. Figure~\ref{fig:fm2-pop} is consistent
with this prediction. On both systems, the final validation MSE stays above
$\sigma^2_{\mathrm{phase}}$, given a large $W$, matching the FM1
results. The best validation MSE during training is closer to the floor but also
stays above it, which suggests that failing to go below the floor is not
due to insufficient training. Detailed $W$ sweep results
are given in Tables~\ref{tab:fm2-goodwin} and~\ref{tab:fm2-gold95}.

\textbf{Population error at practical window lengths.}
Increasing $W$ extends the observation window for the same periodic
trajectories; it does not require increasingly long periods in the
dataset. Although the bound is asymptotic, the OOD-split ablations
in Table~\ref{tab:ablation_both} show substantial errors for
phase-aligned, fixed-window targets even when $W$ equals the
largest period in the training set.

\begin{figure}[t]
    \centering
    \includegraphics[width=0.9\linewidth]{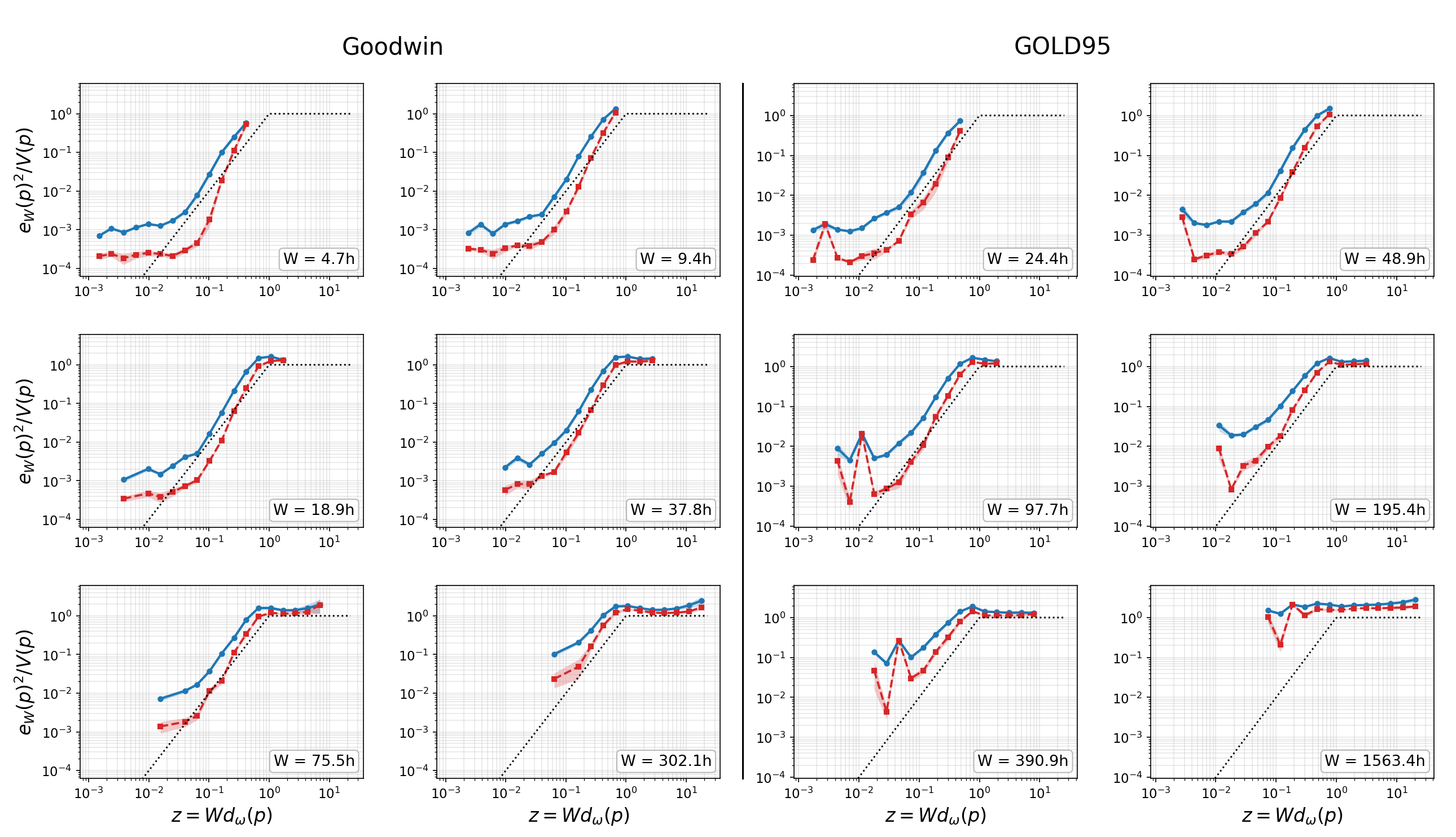}
\caption{Per-sample OOD error for Goodwin (left) and
\textsc{gold95} (right). Blue: bin means; red: bin minima;
shading: one standard deviation across seeds; dotted line:
$\min\{z^2,1\}$.}
    \label{fig:fm2-collapse}
\end{figure}

\textbf{Per-sample OOD error behavior.} Theorem~\ref{thm:fm2_pointwise}
gives the asymptotic lower bound
$e_W(p)^2/V(p)\geq \min\{z^2,1\}-o(1)$ near the training
frequency boundary, where $z=Wd_\omega(p)$ measures accumulated
phase drift. We partition $z$ into logarithmic bins and compute
the mean and minimum of $e_W(p)^2/V(p)$ in each bin for each seed.
Figure~\ref{fig:fm2-collapse} shows a trend consistent with the
asymptotic bound. At every tested $W$, the bin mean lies above
$\min\{z^2,1\}$. The bin minimum falls below this curve at smaller
windows, but lies above it for the large tested windows starting at
$W=302.1$ h for Goodwin and $W=195.4$ h for \textsc{gold95}.
The gap above the curve generally widens as $W$ increases.
This pattern recurs across seeds: the relative standard deviation
across seeds is $2\%$--$12\%$ for the bin mean and $8\%$--$22\%$
for the bin minimum. The details of the bin partition and relative standard
deviations are given in Appendix~\ref{app:fm2_per_sample}.

\subsection{Validation across architectures}
\label{sec:other-arch}

The preceding sweeps use only the FNO, leaving open whether the observed
error-floor behavior depends on the choice of architecture. To test this,
we repeat the shift and window sweeps with a one-dimensional CNN and a
multilayer perceptron (MLP). We use the same data and sweep values for
$c$ and $W$, with five training seeds per setting.
Appendix~\ref{app:other-model} gives the setups and figures.

The two newly tested architectures reproduce results consistent with the
predicted error floors. Similar trends are expected because the amplified
phase variation arises from the targets, and the asymptotic floor in
Theorem~\ref{thm:unified} is shared by architectures satisfying its assumptions.
As in the FNO experiments, the best and final validation
MSEs for the CNN and MLP exceed the population floor
$\sigma^2_{\mathrm{phase}}$ on both systems
at large values of $c$ and $W$. In the $c$ sweep, at large shifts, the best validation MSE,
averaged over five seeds, is within about $2.2\%$ of the floor.
In both the $c$ and $W$ sweeps, the final validation MSEs
remain above the floor at large shifts and windows, respectively. At every $W$, the mean normalized squared error in each
bin lies above the reference curve $\min\{z^2,1\}$ from
Theorem~\ref{thm:fm2_pointwise}. At the largest tested window for each
system, the minimum error in each bin also lies above the curve.
Details and figures are provided in Appendix~\ref{app:other-model}.
\section{Escaping the Error Floor: Target Design Guidance}
\label{sec:decoupled}
Theorem~\ref{thm:unified} gives guidance on designing targets that avoid
the identified class of failures. For example, the period--waveform
decomposition discussed in the related work sets $\beta=0$ in
\eqref{eq:coordinate_target}, removing the dependence of phase on
$\lambda$. It learns a waveform on normalized time and its period
separately.

Section~\ref{sec:failure_modes} has already demonstrated the empirical
efficacy of this representation. In this section, we provide a direct
positive result by proving its neural-network approximability under
the stated dynamical-system assumptions. We also show how to combine
the learned maps to reconstruct accurate
physical-time trajectories given regularity of limit cycle.

\subsection{Decoupled representation}
\label{sec:decoupled_method}
The decoupled target learns the anchored waveform
$U(s;p)$ and period $T(p)$ separately.  In \eqref{eq:coordinate_target}, its waveform component is exactly
the choice $\alpha(s)=s$ and $\beta=0$, so the large phase term is absent.
For each ODE sample, we remove the transient, anchor phase at the unique
maximum of a scalar observable $a_{\mathrm{obs}}$, chosen as one state
coordinate in our ODE experiments, and sample one cycle at fixed phases
$\tau_1,\ldots,\tau_G$. With
$U_G(p)=(U(\tau_1;p),\ldots,U(\tau_G;p))$, the two predictors are
\begin{equation}
 \widehat T(p)=\mathcal M_\psi(p),\qquad
 \widehat U_G(p)=\mathcal F_\vartheta(p),
 \label{eq:decoupled_maps}
\end{equation}
where $\mathcal M_\psi$ and $\mathcal F_\vartheta$ denote the period and
waveform predictors. A periodic interpolant supplies values between phase samples. Neither target
contains an uncontrolled observation phase. 

\subsection{ODE regularity and approximability}
\label{sec:continuity_main}
The following conditions ensure that the period and anchored waveform of
\eqref{eq:parametric_ode} are regular; they are independent of the population
assumptions used for the error floors.
\begin{assumption}[Regular periodic setting]\label{ass:decouple_regular}
The vector field satisfies $f\in C^2(\mathcal D\times\mathcal O;\R^n)$,
and trajectories remain in $\mathcal D$. We assume $\mathcal P$ is compact. For every $p\in\mathcal P$, the
ODE has exactly one limit cycle $\mathcal C(p)\subset\mathcal D$, which
is attracting and hyperbolic with minimal period $T(p)>0$, and no
bifurcation, loss of hyperbolicity, or branch switching occurs over
$\mathcal P$. The observable $a_{\mathrm{obs}}\in C^2(\mathcal D;\R)$
attains its maximum at exactly one phase per cycle, where its first time
derivative vanishes and its second time derivative is strictly negative.
\end{assumption}

\begin{theorem}[Approximability of decoupled ODE targets]
\label{thm:decoupled_continuity}
Under Assumption~\ref{ass:decouple_regular}, $T$ and
$U:\Torus\times\mathcal P\to\R^n$ are continuous. Hence $U_G$ is continuous,
and both $T$ and $U_G$ are uniformly approximable by feedforward networks
with suitable continuous nonpolynomial activations.
\end{theorem}
\noindent\textit{Proof sketch.}
Hyperbolicity continues the periodic orbit and its return time locally in
$p$ \citep[Lemma~6.9 and~12.7]{teschl2012ordinary}. The implicit function theorem continues the nondegenerate maximizing
phase, while uniqueness makes the local anchors consistent. Thus $T$ and the
anchored waveform are continuous, and compactness gives the
approximability \citep{hornik1989multilayer,leshno1993multilayer}. Appendix~\ref{app:continuity_proofs} gives the proof details.

\subsection{Physical-time reconstruction via decoupled targets}
\label{sec:snapshot}
Models for the two decoupled targets are straightforward to train and achieve
small errors. For two ODE systems, the waveform errors are reported in Table~\ref{tab:ablation_both},
while period errors and training details are given in
Appendix~\ref{app:main_training_setup}. The remaining question is how to use
these predictions to reconstruct a physical-time trajectory.

The learned period and waveform determine the periodic orbit but not the
phase, so a single full-state snapshot \(y_{\mathrm{obs}}\) suffices to
recover a trajectory. An interpolant maps $\widehat U_G$ to $\widehat U$ on $[0,1)$ between
the sampled phases $\tau_1,\ldots,\tau_G$. We estimate
\(\widehat\phi\in\argmin_{\phi\in[0,1)}
\norm{\widehat U(\phi;p)-y_{\mathrm{obs}}}_X\) and reconstruct
\(\widehat x(t)=\widehat U(\widehat\phi+t/\widehat T(p);p)\),
with the waveform extended periodically.
In the exact setting, the phase is unique because a nonconstant
minimal-period ODE orbit cannot visit the same state at two distinct phases.
In practice, we match the snapshot to the predicted waveform in the
transformed state space used for training.

As shown in Table~\ref{tab:snapshot_results}, a single snapshot enables
accurate reconstruction on both systems over one true period.
For each validation sample, we take a snapshot at a randomly sampled
post-burn-in time and generate the reference trajectory by reintegrating
the ODE from that snapshot using the sample's parameters.
The initial phase is estimated by a coarse-to-fine grid search,
and the predicted waveform and period are combined to reconstruct
the trajectory for comparison with the reference at matching physical times.
Over five true periods, IID reconstruction errors remain small,
while OOD errors increase substantially.
This physical-time reconstruction task differs from the phase-aligned
waveform prediction in Table~\ref{tab:ablation_both},
so their MSEs are not directly comparable.
Implementation details, metrics, and per-seed results are provided in
Appendix~\ref{app:snapshot_details}.

\begin{table}[htbp]
\centering\small
\setlength{\tabcolsep}{4pt}
\caption{Physical-time reconstruction over one and five true periods.
Period relative MSE is the sample mean of $(\widehat T-T)^2/T^2$.
Entries are mean $\pm$ standard deviation across five training seeds.}
\label{tab:snapshot_results}
\begin{tabular}{llccc}
\toprule
& & \multicolumn{2}{c}{Reconstruction MSE} & \\
\cmidrule(lr){3-4}
System & Split & One period & Five periods & Period relative MSE\\
\midrule
Goodwin & IID
& $(4.10\pm0.63){\times}10^{-4}$
& $(5.78\pm0.63){\times}10^{-4}$
& $(2.61\pm0.23){\times}10^{-6}$\\
& OOD
& $(3.61\pm0.56){\times}10^{-2}$
& $(6.0\pm1.2){\times}10^{-1}$
& $(1.86\pm0.42){\times}10^{-3}$\\
\midrule
Gold95 & IID
& $(1.330\pm0.033){\times}10^{-3}$
& $(6.78\pm0.53){\times}10^{-3}$
& $(3.66\pm0.29){\times}10^{-5}$\\
& OOD
& $(3.84\pm0.32){\times}10^{-2}$
& $(4.80\pm0.30){\times}10^{-1}$
& $(1.253\pm0.082){\times}10^{-3}$\\
\bottomrule
\end{tabular}
\end{table}
\section{Beyond ODEs: a FitzHugh--Nagumo PDE Case Study}
\label{sec:fhn}

We next examine the two failure modes in a more challenging setting:
temporally periodic solutions of a partial differential equation (PDE). Taking $X=L^2(\mathcal D_x;\R^n)$ extends the FM1 and FM2 results directly to
periodic PDE fields satisfying the assumptions of Section~\ref{sec:theory}.
We test this setting on a winding-one traveling-wave branch of the
one-dimensional FitzHugh--Nagumo system adapted from
\citet{cebrian2024six}, with periodic boundary conditions. We align the waves by starting each one-period record at the time when the
first spatial Fourier coefficient of \(u\) has zero phase. Further details are given in
Appendix~\ref{app:fhn_details}.

\begin{table}[htbp]
\caption{FitzHugh--Nagumo PDE representation ablations.}
\label{tab:fhn_results}
\centering
\small
\setlength{\tabcolsep}{3.2pt}
\begin{tabular}{lccccc}
\toprule
 & Norm. & Align & Split & Train MSE & Val MSE \\
\midrule
Main & \cmark & \cmark & IID
& $9.65\!\times\!10^{-7}$ & $9.25\!\times\!10^{-7}$ \\
1 & \xmark & \cmark & IID
& $7.74\!\times\!10^{-6}$ & $7.61\!\times\!10^{-6}$ \\
2 & \cmark & \xmark & IID
& $9.80\!\times\!10^{-1}$ & $1.00$ \\
3 & \xmark & \xmark & IID
& $9.85\!\times\!10^{-1}$ & $1.00$ \\
4 & \cmark & \cmark & OOD
& $3.07\!\times\!10^{-6}$ & $1.18\!\times\!10^{-2}$ \\
5 & \xmark & \cmark & OOD
& $5.63\!\times\!10^{-6}$ & $2.66\!\times\!10^{-1}$ \\
\bottomrule
\end{tabular}
\end{table}

The results match the qualitative behavior of the ODE ablations in
Table~\ref{tab:ablation_both}: both failure modes persist for spatial fields,
while the decoupled method accurately learns the periodic solution.
Extending the continuity theory for the decoupled representation to PDEs
requires a separate treatment of traveling-wave branches and translation
symmetry, which we leave for future work.
\section{Discussion and Limitations}
\label{sec:discussion}

\textbf{Regularity of the target maps.}
The error-floor analysis needs no explicit governing equation: it applies to
any parameterized periodic function whose waveform and period maps satisfy
the stated regularity, provided frequency varies with the parameters and the
predictor's parameter Lipschitz constant grows more slowly than the time
scale. The ODE result gives sufficient conditions for continuity of the
decoupled targets, which can fail at bifurcations \citep[Chapters 5]{Kuznetsov2023}.
Analogous regularity for PDE systems remains future work.

\textbf{OOD prediction remains difficult.}
Decoupling reduces OOD error in our experiments but does not eliminate it;
waveform errors stay larger than their iid counterparts. Moreover, parameter
extrapolation remains a general difficulty in operator learning
\citep{zhu2023reliable}.

\textbf{Structure-specific decoupled target.}
Our method assumes the period is determined only by the parameter. More
complex periodic structure will need more than this simple escape. In
atrial fibrillation, for example, period length varies by region of the
heart \citep{doi:10.1161/01.CIR.98.12.1236}, so the period is no longer a
scalar. Such settings will need decoupled targets built for their specific
structure.

\textbf{Operator learning beyond trajectory targets.}
The operator to be learned need not map parameters directly to trajectories
on a domain: other targets include steady states
\citep{marwah2023deep,zhang2023neural}, Green's functions for PDE
parameters \citep{melchers2026neural}, and Newton solvers for nonlinear
PDEs \citep{hao2024newton}. Future work could design operator-learning targets for a wider range of physical quantities.

\section{Conclusion}
\label{sec:conclusion}

The main lesson of this paper is that, for periodic systems, the error can
be determined mainly by the learning target rather than by training.
A periodic orbit can depend regularly on the parameters while the
trajectory recorded from it becomes increasingly sensitive to them.
Under our assumptions, learning the target requires a predictor whose
sensitivity to the parameters grows comparably with the time scale.
If the sensitivity grows more slowly, the error floor remains.
Long transient removal and wide observation windows are two examples
within the family of coordinate choices studied here.
Such choices are usually treated as preprocessing details.
Our results suggest they deserve the same scrutiny as the model,
and that errors reported on periodic benchmarks may partly reflect
these choices.

Our theorem can guide the design of targets that avoid the identified
class of error floors. We revisit the established period--waveform
decomposition, showing both how it avoids this failure mechanism and
why it is neural-network approximable under our assumptions.
This representation uses the fact that the initial phase is an
observational choice, while the period and anchored waveform vary
regularly with the system parameters.
More broadly, the properties of a dynamical system may help us design
better learning targets. Exploring this approach for systems with
other properties, including nonperiodic systems, is a direction for
future work.
\bibliography{ref}

\clearpage
\appendix
\section{ODE Systems and Data Generation}
\label{app:odes}

\subsection{Goodwin oscillator}
\label{app:goodwin}

The Goodwin oscillator~\citep{gonze2021goodwin} is
\begin{equation}
\frac{dX}{dt}=\frac{a_1}{1+(Z/K)^{n_H}}-b_1X,
\qquad
\frac{dY}{dt}=\alpha_1X-\beta_1Y,
\qquad
\frac{dZ}{dt}=\gamma_1Y-\delta_1Z,
\end{equation}
with parameter vector
\begin{equation}
    p=(a_1,K,b_1,\alpha_1,\beta_1,\gamma_1,\delta_1,n_H)\in\R^8.
\end{equation}
The \(X\) coordinate is used for peak anchoring.

\subsection{Goldbeter 1995 circadian model}
\label{app:gold95}

The \textsc{gold95} state is \(x=(M,P_0,P_1,P_2,P_N)\), representing
\emph{per} mRNA, unphosphorylated, mono- and bisphosphorylated cytosolic PER,
and nuclear PER~\citep{goldbeter1995model}. Its equations are
\begin{align}
\frac{dM}{dt}
&=v_s\frac{K_I^{n_H}}{K_I^{n_H}+P_N^{n_H}}-v_m\frac{M}{K_m+M},\\
\frac{dP_0}{dt}
&=k_sM-V_1\frac{P_0}{K+P_0}+V_2\frac{P_1}{K+P_1},\\
\frac{dP_1}{dt}
&=V_1\frac{P_0}{K+P_0}-V_2\frac{P_1}{K+P_1}
  -V_3\frac{P_1}{K+P_1}+V_4\frac{P_2}{K+P_2},\\
\frac{dP_2}{dt}
&=V_3\frac{P_1}{K+P_1}-V_4\frac{P_2}{K+P_2}
  -k_1P_2+k_2P_N-v_d\frac{P_2}{K_d+P_2},\\
\frac{dP_N}{dt}
&=k_1P_2-k_2P_N.
\end{align}
Following~\citet{goldbeter1995model}, the four phosphorylation Michaelis
constants are set equal to \(K\), and the Hill coefficient is fixed at
\(n_H=4\). The free parameter vector is
\begin{equation}
p=(v_s,v_m,K_m,k_s,v_d,K_d,k_1,k_2,K_I,K,V_1,V_2,V_3,V_4)\in\R^{14}.
\end{equation}
The \(M\) coordinate is used for peak anchoring.

\subsection{Sampling, integration, and acceptance}
\label{app:datagen}

Candidate parameter vectors are drawn from a scrambled Sobol sequence over the
box in Table~\ref{tab:sobol_bounds}. Each ODE is integrated until transient
dynamics decay, the asymptotic period is estimated from successive peaks, and
one peak-anchored period is resampled to \(128\) uniform phase points with
refinement factor \(5\). Candidates are rejected if integration fails, no
sustained oscillation is detected, or period detection fails.

For Goodwin, \(5{,}000\) trajectories are accepted from \(5{,}062\) candidates
(\(1.2\%\) rejection), with \(T\in[3.62,7.28]\) h and mean \(T=4.94\) h. For
\textsc{gold95}, \(10{,}000\) are accepted from \(13{,}958\) candidates
(\(28.4\%\) rejection), with \(T\in[16.01,40.71]\) h and mean \(T=25.05\) h.
The \textsc{gold95} rejections comprise \(2{,}318\) period-detection failures
and \(1{,}640\) oscillation-criterion failures.

Solver and tolerance settings, burn-in horizons, initial conditions,
peak-detection and oscillation thresholds, and all random seeds are specified
in the released data-generation code.

\begin{table}[htbp]
\caption{Sobol sampling box. Every interval is \(\pm30\%\) of the baseline
value, except the Goodwin Hill exponent, which is sampled on
\(n_H\in[10,14]\).}
\label{tab:sobol_bounds}
\centering
\small
\setlength{\tabcolsep}{4.5pt}
\begin{tabular}{lrrr@{\hskip 2.2em}lrrr}
\toprule
\multicolumn{4}{c}{Goodwin} & \multicolumn{4}{c}{\textsc{gold95}} \\
\cmidrule(r){1-4}\cmidrule(l){5-8}
Param & Baseline & Lower & Upper & Param & Baseline & Lower & Upper \\
\midrule
\(a_1\)      & 8.37 & 5.86 & 10.88 & \(v_s\) & 0.76 & 0.53 & 0.99 \\
\(K\)        & 1.37 & 0.96 & 1.78  & \(v_m\) & 0.65 & 0.46 & 0.85 \\
\(b_1\)      & 1.00 & 0.70 & 1.30  & \(K_m\) & 0.50 & 0.35 & 0.65 \\
\(\alpha_1\) & 1.00 & 0.70 & 1.30  & \(k_s\) & 0.38 & 0.27 & 0.49 \\
\(\beta_1\)  & 0.60 & 0.42 & 0.78  & \(v_d\) & 0.95 & 0.67 & 1.24 \\
\(\gamma_1\) & 1.00 & 0.70 & 1.30  & \(K_d\) & 0.20 & 0.14 & 0.26 \\
\(\delta_1\) & 0.80 & 0.56 & 1.04  & \(k_1\) & 1.90 & 1.33 & 2.47 \\
\(n_H\)      & 12.0 & 10.0 & 14.0  & \(k_2\) & 1.30 & 0.91 & 1.69 \\
             &      &      &       & \(K_I\) & 1.00 & 0.70 & 1.30 \\
             &      &      &       & \(K\)   & 2.00 & 1.40 & 2.60 \\
             &      &      &       & \(V_1\) & 3.20 & 2.24 & 4.16 \\
             &      &      &       & \(V_2\) & 1.58 & 1.11 & 2.05 \\
             &      &      &       & \(V_3\) & 5.00 & 3.50 & 6.50 \\
             &      &      &       & \(V_4\) & 2.50 & 1.75 & 3.25 \\
\bottomrule
\end{tabular}
\end{table}
\clearpage
\subsection{Period distribution}
Figure~\ref{fig:period_distributions} shows the period distributions
for Goodwin and \textsc{gold95} across all samples in each dataset. In both systems, the period distribution has one dominant peak and a longer right tail.

\begin{figure}[htbp]
    \centering
    \includegraphics[width=0.9\linewidth]{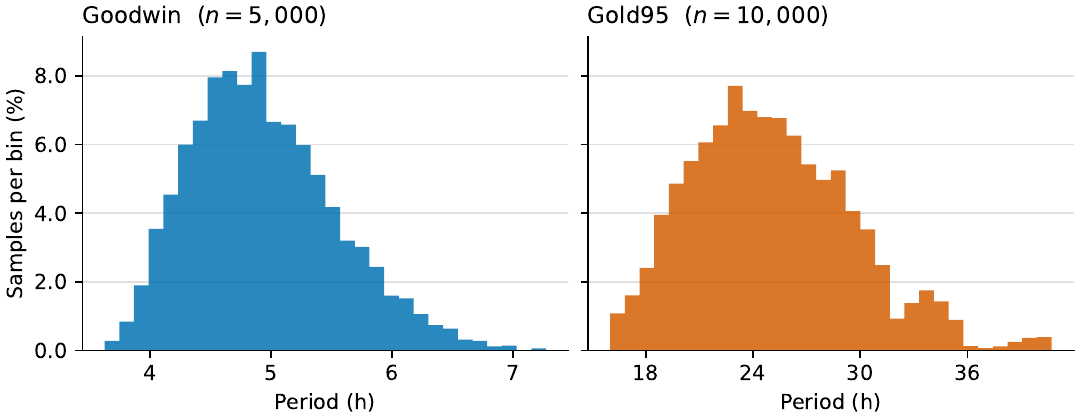}
    \caption{Period distributions for Goodwin and \textsc{gold95}.
    Training and validation samples are pooled within each system.
    Bar heights give the percentage of that system's samples in each period bin.}
    \label{fig:period_distributions}
\end{figure}
\section{Training Setup and Ablation Constructions}
\label{app:training_ablations}
\label{app:ablation_setup}

\subsection{Main decoupled training setup}
\label{app:main_training_setup}

For each system, the main model trains a period MLP and a waveform FNO
separately. We use a fixed \(80/20\) train--validation split generated with
split seed \(0\), giving \(4{,}000/1{,}000\) samples for Goodwin and
\(8{,}000/2{,}000\) samples for \textsc{gold95}. Each model is trained with
training seeds \(0,1,2,3,4\); these seeds affect weight initialization and
batch shuffling but not the data split. Unless stated otherwise, the ablation
experiments use the same model capacity and output grid.

\begin{table}[H]
\caption{Default configuration for the period and waveform models.}
\label{tab:main_training_setup}
\centering
\small
\setlength{\tabcolsep}{4pt}
\begin{tabular}{lccccc}
\toprule
Model & Architecture & Grid & Epochs & LR & Min. LR \\
\midrule
Period MLP
& depth 2, width 128
& --
& 100
& $5\!\times\!10^{-4}$
& $10^{-6}$ \\
Waveform FNO
& 16 modes, 128 channels
& 128
& 100
& $10^{-3}$
& $10^{-5}$ \\
\bottomrule
\end{tabular}
\end{table}

Across the five training seeds, the period MLP has mean IID validation RMSE
\(0.0077\) h for Goodwin and \(0.165\) h for \textsc{gold95}.

All models are trained with Adam, cosine learning-rate annealing, batch size
\(32\), no weight decay, and no dropout. For the period MLP, the parameters
and periods are standardized. For the waveform FNO, the input coordinates
and parameters are standardized, while non-negative state variables are
power-transformed. All scalers are fitted on the training split only.

Table~\ref{tab:ablation_both} reports losses at the final training epoch.
The mean, standard deviation, minimum, and maximum validation losses are
computed across the five training seeds. Remaining implementation details
are provided in the released code.

\subsection{Ablation implementation}
\label{app:ablation_implementation}

Unless stated otherwise, all models use the default training settings and
the \(128\)-point grid in Table~\ref{tab:main_training_setup}. Each experiment
is trained for \(100\) epochs using training seeds \(0,1,2,3,4\). Each seed controls model weight initialization (\texttt{torch.manual\_seed}). The IID
experiments use the same \(80/20\) split generated with split seed \(0\).
The period-based OOD splits are fixed across training seeds.
Table~\ref{tab:ablation_key} maps the original experiment numbers to their
preprocessing choices and validation splits.

\begin{table}[H]
\caption{Constructions used in the ODE ablations.}
\label{tab:ablation_key}
\centering
\small
\setlength{\tabcolsep}{5pt}
\begin{tabular}{lcccl}
\toprule
Row & Norm. & Align. & Split & Waveform target \\
\midrule
Main & $\checkmark$ & $\checkmark$ & IID & Aligned cycle \\
1 & $\times$ & $\checkmark$ & IID & Aligned fixed window \\
2 & $\checkmark$ & $\times$ & IID & Unaligned cycle \\
3 & $\times$ & $\times$ & IID & Unaligned fixed window \\
4 & $\checkmark$ & $\checkmark$ & OOD & Aligned cycle \\
5 & $\times$ & $\checkmark$ & OOD & Aligned fixed window \\
\bottomrule
\end{tabular}
\end{table}

\paragraph{Aligned fixed window (row 1).}
For Goodwin, each trajectory is integrated from the same fixed initial
condition for \(500\) h. The first subsequent peak is located, and the ODE is
then integrated directly over a \(10\) h window. For \textsc{gold95}, the
peak state from the existing aligned waveform is used as the initial
condition, and the ODE is integrated directly over a \(100\) h window without
an additional burn-in. Both trajectories are sampled at \(128\) uniform
points without periodic tiling. The FNO input is \((t/W,p)\), so phase
alignment is retained but time is not period-normalized.

\paragraph{Unaligned cycle (row 2).}
For Goodwin, every parameterized system starts from the same fixed initial
condition and is integrated for \(t_{\mathrm{burn}}=500\) h. One period is
then recorded directly from the integration endpoint. Its initial phase
therefore has a portion from the parameter-dependent transient. For
\textsc{gold95}, each peak-aligned state is instead advanced by the same
\(500\) h, giving the controlled phase shift
\(\phi_0(p)=t_{\mathrm{burn}}/T(p)\pmod 1\), after which one period is recorded
without realignment. In both systems, the recorded trajectory is sampled on
\(\tau=t/T(p)\in[0,1]\). The target is period-normalized but not
phase-aligned.

\paragraph{Unaligned fixed window (row 3).}
The unaligned one-period trajectories from row 2 are first interpolated
using periodic cubic splines. The resulting continuous waveforms are then
repeated over a fixed physical-time window and sampled at \(128\) uniform
points. We use \(W=10\) h for Goodwin and \(W=100\) h for \textsc{gold95}.
The FNO input is \((t/W,p)\), with neither phase alignment nor period
normalization.

\paragraph{Aligned cycle under the OOD split (row 4).}
The targets are peak-aligned, period-normalized waveforms. For Goodwin, the
\(2{,}000\) shortest-period samples form the training set
(\(T\in[3.62,4.72]\) h), and the \(500\) longest-period samples form the
validation set (\(T\in[5.77,7.28]\) h). For \textsc{gold95}, training uses
\(T<0.6T_{\max}=24.43\) h (\(n_{\mathrm{train}}=4{,}822\)), whereas
validation uses \(T\geq24.43\) h (\(n_{\mathrm{val}}=5{,}178\),
\(T\in[24.44,40.71]\) h).

\paragraph{Aligned fixed window under the OOD split (row 5).}
The period splits are identical to row 4. The window is the largest training
period: \(W=4.72\) h for Goodwin and \(W=24.43\) h for \textsc{gold95}.
Using linear interpolation, training trajectories are tiled to fill the
window, whereas validation trajectories are clipped to their first \(W\)
hours. The validation targets cover \(0.65\)--\(0.82\) of a cycle for Goodwin
and \(0.60\)--\(1.00\) cycle for \textsc{gold95}. Phase alignment is retained.
\section{Dynamical Systems Terminology}
\label{app:dynamical_systems_terminology}
For readers whose background is primarily in machine learning,
we briefly explain the dynamical system terminology used in the proofs in
Sections~\ref{app:unified}, ~\ref{app:pointwise}, and~\ref{app:continuity_proofs},
focusing on its meaning in our study.

\paragraph{Autonomous ODE and vector field.}
We focus on systems of the form $\dot{x}=f(x,p)$, where the
evolution rule $f$ has no explicit dependence on time.
Such systems are called autonomous. The function $f$ is called
the vector field and gives the instantaneous rate of change
of the state $x$ for a fixed parameter $p$.

\paragraph{Attracting limit cycle.}
For a fixed parameter $p$, $\mathcal C(p)$ is an isolated periodic
orbit of $\dot{x}=f(x,p)$. It is attracting if trajectories starting
sufficiently close to $\mathcal C(p)$ approach it as time increases. In our study, we also assume it is unique given a fixed parameter.

\paragraph{Minimal period.}
$T(p)$ is the smallest positive time needed for a trajectory on
$\mathcal C(p)$ to return to its starting state.

\paragraph{Flow.}
$\Phi(t,x,p)$ is the state reached after evolving the ODE for time
$t$ from initial state $x$ at parameter $p$.

\paragraph{Transverse section.}
In the $n$-dimensional state space, we choose a smooth
$(n-1)$-dimensional surface $\Sigma$ through a point $x_0$ on
$\mathcal C(p_0)$. The section is transverse if the velocity
$f(x_0,p_0)$ has a nonzero component normal to $\Sigma$ at $x_0$.

\paragraph{Poincar\'e map and first return time.}
For $x\in\Sigma$ near $x_0$ and $p$ near $p_0$, the
\emph{first return time} $\tau_\Sigma(x,p)$ is the first positive
time at which the trajectory starting at $x$ meets $\Sigma$ again.
For fixed $p$, the \emph{Poincar\'e map} sends the starting point $x$
to the point $\Phi(\tau_\Sigma(x,p),x,p)$ where the trajectory
returns to $\Sigma$. The point $x_\Sigma(p)$ where the nearby
cycle crosses $\Sigma$ is a fixed point of this map.

\paragraph{Hyperbolicity.}
The cycle $\mathcal C(p)$ is hyperbolic if the Jacobian of the
Poincar\'e map at $x_\Sigma(p)$ has no eigenvalues on the unit
circle.

\paragraph{Local $C^k$ continuation.}
A \emph{local continuation} extends the assignment $p_0\mapsto x_0$
to a map $p\mapsto x_\Sigma(p)$ defined for $p$ near $p_0$.
For each input $p$, the output $x_\Sigma(p)$ is a fixed point of
the Poincar\'e map at that parameter. The continuation is $C^k$
if the map $p\mapsto x_\Sigma(p)$ has continuous derivatives up
to order $k$.

\paragraph{Scalar observable and phase anchor.}
$a_{\mathrm{obs}}(x)$ is the scalar quantity used to select phase
zero. In our ODE experiments, it is one component of the state
$x$. We define $U(0;p)$ as the unique point on $\mathcal C(p)$
where $a_{\mathrm{obs}}$ attains its maximum.

\paragraph{Nondegenerate maximum.}
Assuming function $g$ is twice-differentiable, a maximum at $s_*$ is nondegenerate if $g'(s_*)=0$ and
$g''(s_*)<0$. The first-order term vanishes, so a small input
perturbation lowers $g$ quadratically. We assume these conditions
at $\tau=0$ for $g(\tau)=a_{\mathrm{obs}}(U(\tau;p))$.
\section{Unified Phase Averaging}
\label{app:unified}

\subsection{Setup and lemmas}

Recall that \(\mathcal H=L^2([0,1];X)\), and define
\[
\bar U(p)=\int_0^1U(\tau;p)\,d\tau,
\qquad
V(p)=\int_0^1
\norm{U(\tau;p)-\bar U(p)}_X^2\,d\tau.
\]
The population phase variance is
\[
\sigma_{\mathrm{phase}}^2
=
\int_{\mathcal P}V(p)\,d\mu(p).
\]

We first give an elementary periodic-averaging estimate.

\begin{lemma}[Average of a zero-mean periodic function]
\label{lem:periodic_average}
Let \(Z\) be a Banach space and let \(g:\Torus\to Z\) be
continuous and satisfy
\[
\int_0^1g(\tau)\,d\tau=0.
\]
Then, for every \(\theta\in\Torus\), \(q>0\), and
\(\xi\in\mathbb R\),
\[
\left\|
\int_0^1g(\theta+qz\xi)\,dz
\right\|_Z
\le
\sup_{\tau\in\Torus}\norm{g(\tau)}_Z
\min\{1,(q|\xi|)^{-1}\},
\]
where the second term in the minimum is interpreted as \(+\infty\)
when \(\xi=0\).
\end{lemma}

\begin{proof}
Set \(G=\sup_{\tau\in\Torus}\norm{g(\tau)}_Z\) and
\(L=q|\xi|\). The integral is bounded by \(G\), which proves the
claim when \(L<1\). If \(L\ge1\), write \(L=n+\rho\), where
\(n\) is an integer and \(0\le\rho<1\). Changing variables and
canceling the \(n\) complete periods gives
\[
\left\|\int_0^1g(\theta+qz\xi)\,dz\right\|_Z
=\frac1L\left\|\int_{\theta_0+n}^{\theta_0+n+\rho}g(t)\,dt\right\|_Z
\le\frac{\rho G}{L}\le\frac GL,
\]
where \(\theta_0=\theta\) if \(\xi>0\), and
\(\theta_0=\theta-L\) if \(\xi<0\).
\end{proof}

The following direct consequence gives the two estimates used below.

\begin{corollary}[Uniform phase-averaging bounds]
\label{cor:uniform_phase_average}
Under the assumptions of Theorem~\ref{thm:unified}, let
\[
v(\tau;p)=U(\tau;p)-\bar U(p),
\qquad M=\sup_{\tau,p}\norm{U(\tau;p)}_X.
\]
For \(q>0\), define
\[
\eta_q(s)=
\begin{cases}
\min\{1,(q|\beta(s)|)^{-1}\},&\beta(s)\ne0,\\
1,&\beta(s)=0,
\end{cases}
\qquad J_\beta(q)=\int_0^1\eta_q(s)\,ds.
\]
Then \(J_\beta(q)\to0\) as \(q\to\infty\). Moreover, for every
\(p\in\mathcal P\) and every measurable
\(\theta:[0,1]\to\Torus\),
\begin{align}
\left\|\int_0^1v\bigl(\theta(\cdot)+qz\beta(\cdot);p\bigr)\,dz
\right\|_{\mathcal H}
&\le2M J_\beta(q)^{1/2},
\label{eq:uniform_waveform_average}\\
\left|\int_0^1\int_0^1
\left[\norm{v(\theta(s)+qz\beta(s);p)}_X^2-V(p)\right]\,dz\,ds\right|
&\le4M^2J_\beta(q).
\label{eq:uniform_squared_norm_average}
\end{align}
\end{corollary}

\begin{proof}
Since \(\beta\ne0\) almost everywhere,
\(\eta_q(s)\to0\) for almost every \(s\). Since
\(0\le\eta_q\le1\), dominated convergence gives
\(J_\beta(q)\to0\).

For fixed \(p\), both \(v(\cdot;p)\) and
\(\norm{v(\cdot;p)}_X^2-V(p)\) have zero mean, and their norms are
bounded by \(2M\) and \(4M^2\), respectively. Apply
Lemma~\ref{lem:periodic_average} to each function for fixed \(s\).
The resulting pointwise bounds are \(2M\eta_q(s)\) and
\(4M^2\eta_q(s)\). Squaring the first and using
\(\eta_q^2\le\eta_q\) proves
\eqref{eq:uniform_waveform_average}; integrating the second proves
\eqref{eq:uniform_squared_norm_average}.
\end{proof}

We now use frequency as a local parameter coordinate to apply these
bounds over the parameter domain. We first identify the limiting
variance and then prove that the oscillatory part is asymptotically
orthogonal to any sufficiently slow predictor.

\begin{lemma}[Convergence to the phase variance]
\label{lem:rapid_phase_averaging}
Under the assumptions of Theorem~\ref{thm:unified}, let
\(b_\lambda=F_\lambda-\bar U\). Then
\[
\norm{b_\lambda}_{\mu,2}^2
\longrightarrow
\sigma_{\mathrm{phase}}^2=\int_{\mathcal P}V(p)\,d\mu(p).
\]
\end{lemma}

\begin{proof}
Let \(v=U-\bar U\) and recall that
\(M=\sup_{\tau,p}\norm{U(\tau;p)}_X\), so \(\norm{v}_X\le2M\) and
\(b_\lambda(p)(s)=v\bigl(\alpha(s)+\lambda\omega(p)\beta(s);p\bigr)\).
By Assumption~\ref{ass:regularity}(iii), \(v\) is Lipschitz in \(p\) with constant \(L_v=2K_U\). Define
\[f_\lambda(p)=\norm{b_\lambda(p)}_{\mathcal H}^2-V(p).\] 
Then \(|f_\lambda|\le8M^2\) and
\(\int_{\mathcal P}f_\lambda\,d\mu
=\norm{b_\lambda}_{\mu,2}^2-\sigma_{\mathrm{phase}}^2\), so it
suffices to show that this integral tends to zero. We first derive local result then extend it globally.

\noindent\emph{Local coordinates.}
The set
\(\mathcal R=\{p\in\operatorname{int}\mathcal P:
\nabla\omega(p)\ne0,\ \varpi(p)>0\}\)
has full \(\mu\)-measure. Near each point of \(\mathcal R\), the
inverse function theorem gives (after relabeling coordinates if necessary)
\(p=\Psi(y,r)\in C^{1}\) with \(y=(p_1,\ldots,p_{m-1})\) and \(r=\omega(p)\).
Fix a compact box \(Q=B\times[a,b]\) under new coordinate, let
\(L_\Psi\) be a Lipschitz constant of \(\Psi\) in \(r\) on \(Q\), and
let \(0\le\chi\le1\) be continuous and supported in
\(\Psi(\operatorname{int}Q)\). With
\(\kappa(y,r)=\chi(\Psi(y,r))\,\varpi(\Psi(y,r))\,|\det D\Psi(y,r)|\)
and \(\widetilde f_\lambda=f_\lambda\circ\Psi\),
\[
\int_{\mathcal P}\chi f_\lambda\,d\mu
=\int_Q\kappa\,\widetilde f_\lambda\,dy\,dr.
\]

\noindent\emph{Local limit.}
Partition \([a,b]\) into equal cells \(I_j=[r_j,r_j+h]\) with
\(h=\Theta(\lambda^{-1/2})\), so that \(h\to0\) and
\(\lambda h\to\infty\) as \(\lambda\to\infty\). On \(I_j\), freeze
the parameter but not the phase:
\[
B_{\lambda,j}(y,r)(s)=v\bigl(\alpha(s)+\lambda r\beta(s);p_j(y)\bigr),
\qquad
\phi_{\lambda,j}=\norm{B_{\lambda,j}}_{\mathcal H}^2-V(p_j),
\qquad
p_j(y)=\Psi(y,r_j).
\]
Writing
\(\kappa\widetilde f_\lambda
=\kappa(\widetilde f_\lambda-\phi_{\lambda,j})
+(\kappa-\kappa(\cdot,r_j))\phi_{\lambda,j}
+\kappa(\cdot,r_j)\phi_{\lambda,j}\)
on each cell and summing over cells gives
\(\int_{\mathcal P}\chi f_\lambda\,d\mu
=T_{\lambda,1}+T_{\lambda,2}+T_{\lambda,3}\), where
\[
\begin{aligned}
T_{\lambda,1}&=\sum_j\int_B\int_{I_j}\kappa(y,r)
\bigl(\widetilde f_\lambda-\phi_{\lambda,j}\bigr)(y,r)\,dr\,dy,\\
T_{\lambda,2}&=\sum_j\int_B\int_{I_j}
\bigl(\kappa(y,r)-\kappa(y,r_j)\bigr)\phi_{\lambda,j}(y,r)\,dr\,dy,\\
T_{\lambda,3}&=\sum_j\int_B\kappa(y,r_j)
\int_{I_j}\phi_{\lambda,j}(y,r)\,dr\,dy.
\end{aligned}
\]
We show that each term tends to zero. Below,
\(|Q|=\int_Q dy\,dr\).

For \(T_{\lambda,1}\), let \(r\in I_j\). Apply
\(\bigl|\norm{e}^2-\norm{e'}^2\bigr|\le4M\norm{e-e'}\), valid for
\(\norm e,\norm{e'}\le2M\), to both terms of \(f_\lambda\), and use
\(\norm{\Psi(y,r)-p_j(y)}\le L_\Psi h\). This gives
\(|\widetilde f_\lambda(y,r)-\phi_{\lambda,j}(y,r)|\le8ML_vL_\Psi h\),
so
\[
|T_{\lambda,1}|\le8ML_vL_\Psi h\,\sup_Q\kappa\,|Q|\longrightarrow0,\qquad h\to0.
\]

For \(T_{\lambda,2}\), let \(r\in I_j\). Then
\(|\kappa(y,r)-\kappa(y,r_j)|\le\omega_\kappa(h)\), where
\(\omega_\kappa(h)=\sup\{|\kappa(y,r)-\kappa(y,r')|:
(y,r),(y,r')\in Q,\ |r-r'|\le h\}\). Since
\(|\phi_{\lambda,j}|\le8M^2\),
\[
|T_{\lambda,2}|\le8M^2\,\omega_\kappa(h)\,|Q|\longrightarrow0,
\]
since \(\kappa\) is uniformly continuous on \(Q\) and \(h\to0\).

For \(T_{\lambda,3}\), substitute \(r=r_j+hz\) and apply
\eqref{eq:uniform_squared_norm_average} with
\(\theta=\alpha+\lambda r_j\beta\) and \(q=\lambda h\). This gives
\(\bigl|\int_{I_j}\phi_{\lambda,j}(y,r)\,dr\bigr|\le4M^2hJ_\beta(\lambda h)\),
uniformly in \(y\) and \(j\). Summing over the \((b-a)/h\) cells,
\[
|T_{\lambda,3}|\le4M^2J_\beta(\lambda h)\,\sup_Q\kappa\,|Q|\longrightarrow0,
\]
since \(\lambda h\to\infty\) and \(J_\beta(q)\to0\) as
\(q\to\infty\) by Corollary~\ref{cor:uniform_phase_average}.
Hence \(\int_{\mathcal P}\chi f_\lambda\,d\mu\to0\).

\noindent\emph{Global results.}
For \(\varepsilon>0\), choose a compact \(K\subset\mathcal R\) with
\(\mu(\mathcal P\setminus K)<\varepsilon\), a finite cover of \(K\)
by the coordinates above, and continuous cutoffs \(\chi_i\) supported in
them with \(S=\sum_i\chi_i=1\) on \(K\) and \(0\le S\le1\). The local
limit gives \(\int_{\mathcal P}Sf_\lambda\,d\mu\to0\), and
\(\bigl|\int_{\mathcal P}(1-S)f_\lambda\,d\mu\bigr|\le8M^2\varepsilon\).
Hence
\(\limsup_{\lambda\to\infty}\bigl|\int_{\mathcal P}f_\lambda\,d\mu\bigr|
\le8M^2\varepsilon\) for every \(\varepsilon>0\), which proves the
claim.
\end{proof}

\begin{lemma}[Asymptotically vanishing product]
\label{lem:rapid_phase_orthogonality}
Under the assumptions of Theorem~\ref{thm:unified}, let
\(b_\lambda=F_\lambda-\bar U\). Suppose that
\(\sup_\lambda\norm{a_\lambda}_{\mu,2}<\infty\) and
\[
\norm{a_\lambda(p)-a_\lambda(p')}_{\mathcal H}
\le M_\lambda|p-p'|\quad(p,p'\in\mathcal P),
\qquad M_\lambda=o(\lambda).
\]
Then
\[
\langle a_\lambda,b_\lambda\rangle_{L^2(\mu;\mathcal H)}
\longrightarrow0.
\]
\end{lemma}

\begin{proof}
Write \(v(\tau;p)=U(\tau;p)-\bar U(p)\), and recall that
\(M=\sup_{\tau,p}\norm{U(\tau;p)}_X\). Then
\(\norm{v(\tau;p)}_X\le2M\) and
\(b_\lambda(p)(s)=v(\alpha(s)+\lambda\omega(p)\beta(s);p)\).

\noindent\emph{Local coordinate.}
The regular set
\(\mathcal R=\{p\in\operatorname{int}\mathcal P:
\nabla\omega(p)\ne0,\ \varpi(p)>0\}\)
is open and has full \(\mu\)-measure. For each point in \(p^{*}\in\mathcal R\),
choose a neighborhood contained in \(\mathcal R\). After relabeling
coordinates if necessary, the inverse function theorem gives local
coordinates \(y=(p_1,\ldots,p_{m-1})\), \(r=\omega(p)\), and
\(p=\Psi(y,r)\in C^{1}\).

Choose a compact box \(Q=B\times[a,b]\) within this coordinate neighbor such that
\(p_*\in\Psi(\operatorname{int}Q)\) and
\(\Psi(Q)\subset\mathcal R\), and let \(0\le\chi\le1\) be continuous
and supported in \(\Psi(\operatorname{int}Q)\). Set
\[
\begin{aligned}
w(y,r)&=\varpi(\Psi(y,r))|\det D\Psi(y,r)|,
&\kappa(y,r)&=\chi(\Psi(y,r))w(y,r),\\
\widetilde a_\lambda(y,r)&=a_\lambda(\Psi(y,r)),
&\widetilde b_\lambda(y,r)&=b_\lambda(\Psi(y,r)).
\end{aligned}
\]
Then \(d\mu=w\,dy\,dr\), with
\(0<\inf_Qw\le\sup_Qw<\infty\), and \(\kappa\) is bounded and
uniformly continuous on \(Q\). Let \(L_\Psi\) be a Lipschitz constant
for \(\Psi\) in \(r\), and let \(L_v\) be a uniform parameter
Lipschitz constant for \(v\).

Partition \([a,b]\) into equal cells \(I_j=[r_j,r_j+h]\), choosing
\(h=\Theta([\lambda(1+M_\lambda)]^{-1/2})\).
Since \(M_\lambda=o(\lambda)\), we have
\(h\to0\), \(M_\lambda h\to0\), and \(\lambda h\to\infty\).
Define
\[
\begin{aligned}
a_j(y)&=\frac1h\int_{I_j}\widetilde a_\lambda(y,r)\,dr, p_j(y)=\Psi(y,r_j),\\
B_{\lambda,j}(y,r)(s)
&=v\bigl(\alpha(s)+\lambda r\beta(s);p_j(y)\bigr).
\end{aligned}
\]

\noindent\emph{Local cross term.}
The cross term weighted by \(\chi\) is
\begin{equation}
\begin{aligned}
&\int_{\mathcal P}\chi(p)
\langle a_\lambda(p),b_\lambda(p)\rangle_{\mathcal H}\,d\mu(p)\\
&=
\sum_j\int_B\int_{I_j}\kappa(y,r)
\langle\widetilde a_\lambda(y,r),
\widetilde b_\lambda(y,r)\rangle_{\mathcal H}\,dr\,dy.
\end{aligned}
\label{eq:local_cross_integral}
\end{equation}
On each cell, subtract the frozen integrand:
\begin{equation}
\begin{aligned}
&\kappa(y,r)
\langle\widetilde a_\lambda,\widetilde b_\lambda\rangle_{\mathcal H}
-\kappa(y,r_j)\langle a_j,B_{\lambda,j}\rangle_{\mathcal H}\\
&=\kappa(y,r)
\langle\widetilde a_\lambda-a_j,\widetilde b_\lambda\rangle_{\mathcal H}
+\kappa(y,r)
\langle a_j,\widetilde b_\lambda-B_{\lambda,j}\rangle_{\mathcal H}\\
&\quad+
\bigl(\kappa(y,r)-\kappa(y,r_j)\bigr)
\langle a_j,B_{\lambda,j}\rangle_{\mathcal H}.
\end{aligned}
\label{eq:cross_residual_decomposition}
\end{equation}
Denote the integrals of these three terms over all cells by
\(R_{\lambda,1}\), \(R_{\lambda,2}\), and \(R_{\lambda,3}\),
respectively.

\noindent\emph{First residual.}
The definition of \(a_j\) and the Lipschitz bounds give, for
\(r\in I_j\),
\begin{equation}
\begin{aligned}
&\norm{\widetilde a_\lambda(y,r)-a_j(y)}_{\mathcal H}
=
\left\|
\frac1h\int_{I_j}
\bigl[\widetilde a_\lambda(y,r)
-\widetilde a_\lambda(y,r')\bigr]\,dr'
\right\|_{\mathcal H}\\
\le&
\frac1h\int_{I_j}
\norm{\widetilde a_\lambda(y,r)
-\widetilde a_\lambda(y,r')}_{\mathcal H}\,dr'=
\frac1h\int_{I_j}
\norm{a_\lambda(\Psi(y,r))
-a_\lambda(\Psi(y,r'))}_{\mathcal H}\,dr'\\
\le&
\frac{M_\lambda}{h}\int_{I_j}
|\Psi(y,r)-\Psi(y,r')|\,dr'\le
\frac{M_\lambda L_\Psi}{h}\int_{I_j}|r-r'|\,dr'\\
\le&
\frac{M_\lambda L_\Psi}{h}\,h^2
=
L_\Psi M_\lambda h.
\end{aligned}
\label{eq:cross_freezing_bounds}
\end{equation}
Using \eqref{eq:cross_freezing_bounds} and
\(\norm{\widetilde b_\lambda}_{\mathcal H}\le2M\), we obtain
\begin{equation}
|R_{\lambda,1}|
\le2M\sup_Q|\kappa|\,L_\Psi M_\lambda h\,|Q|
\longrightarrow0.
\label{eq:residual1}
\end{equation}

\noindent\emph{Second residual.}
The parameter Lipschitz bound on \(v\) gives, for \(r\in I_j\),
\begin{equation}
\norm{\widetilde b_\lambda(y,r)-B_{\lambda,j}(y,r)}_{\mathcal H}
\le L_v|\Psi(y,r)-\Psi(y,r_j)|
\le L_vL_\Psi h.
\label{eq:cross_waveform_freezing}
\end{equation}
To bound the second residual, we need to control the integral of
\(\norm{a_j(y)}_{\mathcal H}\) over all parameter cells.
Here the \(\mathcal H\)-norm is taken in \(s\), while the outer
integrals are taken in \(y,r\). Since \(a_j(y)\) does not depend
on \(r\),
\[
\sum_j\int_B\int_{I_j}\norm{a_j(y)}_{\mathcal H}\,dr\,dy
=
\sum_jh\int_B\norm{a_j(y)}_{\mathcal H}\,dy.
\]
By the definition of \(a_j\) and the triangle inequality,
\[
h\norm{a_j(y)}_{\mathcal H}
=
\left\|\int_{I_j}\widetilde a_\lambda(y,r)\,dr\right\|_{\mathcal H}
\le
\int_{I_j}\norm{\widetilde a_\lambda(y,r)}_{\mathcal H}\,dr.
\]
Thus, summing over the cells and using Cauchy--Schwarz,
\begin{equation}
\begin{aligned}
\sum_jh\int_B\norm{a_j(y)}_{\mathcal H}\,dy
&\le
\sum_j\int_B\int_{I_j}
\norm{\widetilde a_\lambda(y,r)}_{\mathcal H}\,dr\,dy\\
&=
\int_Q\norm{\widetilde a_\lambda(y,r)}_{\mathcal H}\,dy\,dr\\
&\le
|Q|^{1/2}
\left(
\int_Q
\frac{\norm{\widetilde a_\lambda(y,r)}_{\mathcal H}^2}{w(y,r)}
w(y,r)\,dy\,dr
\right)^{1/2}\\
&\le
\left(\frac{|Q|}{\inf_Qw}\right)^{1/2}
\sup_\lambda\norm{a_\lambda}_{\mu,2}
<\infty.
\end{aligned}
\label{eq:cross_cell_integral_bound}
\end{equation}
Combining \eqref{eq:cross_waveform_freezing} and
\eqref{eq:cross_cell_integral_bound},
\begin{equation}
|R_{\lambda,2}|
\le\sup_Q|\kappa|\,L_vL_\Psi h
\sum_jh\int_B\norm{a_j(y)}_{\mathcal H}\,dy
\longrightarrow0.
\label{eq:residual2}
\end{equation}

\noindent\emph{Third residual.}
Uniform continuity of \(\kappa\) implies that
\(|\kappa(y,r)-\kappa(y,r_j)|\to0\) with a bound independent of
\(y,j\) for \(r\in I_j\). Since
\(\norm{B_{\lambda,j}}_{\mathcal H}\le2M\),
\begin{equation}
|R_{\lambda,3}|
\le2M
\sup_{\substack{j,\ y\in B\\r\in I_j}}
|\kappa(y,r)-\kappa(y,r_j)|
\sum_jh\int_B\norm{a_j(y)}_{\mathcal H}\,dy
\longrightarrow0
\label{eq:residual3}
\end{equation}
by \eqref{eq:cross_cell_integral_bound}.

\noindent\emph{Frozen term.}
Substituting \(r=r_j+hz\) and applying
Corollary~\ref{cor:uniform_phase_average} with \(q=\lambda h\) and
\(\theta(s)=\alpha(s)+\lambda r_j\beta(s)\) modulo one gives
\begin{equation}
\left\|\int_{I_j}B_{\lambda,j}(y,r)\,dr\right\|_{\mathcal H}
\le2MhJ_\beta(\lambda h)^{1/2}=o(h).
\label{eq:cross_cell_phase_average}
\end{equation}
The bound is independent of \(y,j\). Therefore,
\begin{equation}
\begin{aligned}
&\left|
\sum_j\int_B\kappa(y,r_j)
\left\langle
a_j(y),\int_{I_j}B_{\lambda,j}(y,r)\,dr
\right\rangle_{\mathcal H}\,dy
\right|\\
&\le2M\sup_Q|\kappa|\,J_\beta(\lambda h)^{1/2}
\sum_jh\int_B\norm{a_j(y)}_{\mathcal H}\,dy
\longrightarrow0
\end{aligned}
\label{eq:cross_frozen_term}
\end{equation}
by \eqref{eq:cross_cell_phase_average} and
\eqref{eq:cross_cell_integral_bound}.
All these give
\begin{equation}
\int_{\mathcal P}\chi(p)
\langle a_\lambda(p),b_\lambda(p)\rangle_{\mathcal H}\,d\mu(p)
\longrightarrow0.
\label{eq:local_cross_limit}
\end{equation}

\noindent\emph{From local to global.}
For any \(\varepsilon>0\), inner regularity gives a compact
\(K\subset\mathcal R\) such that
\(\mu(\mathcal P\setminus K)<\varepsilon\).
The sets \(\Psi(\operatorname{int}Q)\) constructed above cover \(K\),
so compactness gives a finite subcover. Choose nonnegative continuous
cutoffs \(\chi_i\), supported in these sets, such that
\(S=\sum_i\chi_i=1\) on \(K\) and \(0\le S\le1\) on \(\mathcal P\).

The term weighted by \(S\) tends to zero by
\eqref{eq:local_cross_limit}. Since \(1-S\) vanishes on \(K\),
Cauchy--Schwarz and \(\norm{b_\lambda}_{\mathcal H}\le2M\) give
\begin{equation}
\begin{aligned}
&\left|
\int_{\mathcal P}(1-S)
\langle a_\lambda,b_\lambda\rangle_{\mathcal H}\,d\mu
\right|\\
&\le
\left(
\int_{\mathcal P\setminus K}\norm{a_\lambda}_{\mathcal H}^2\,d\mu
\right)^{1/2}
\left(
\int_{\mathcal P\setminus K}\norm{b_\lambda}_{\mathcal H}^2\,d\mu
\right)^{1/2}\\
&\le2M\left(\sup_\lambda\norm{a_\lambda}_{\mu,2}\right)
\varepsilon^{1/2}.
\end{aligned}
\label{eq:global_cross_remainder}
\end{equation}
By the triangle inequality, \eqref{eq:local_cross_limit}, and
\eqref{eq:global_cross_remainder},
\[
0\le
\limsup_{\lambda\to\infty}
\left|
\langle a_\lambda,b_\lambda\rangle_{L^2(\mu;\mathcal H)}
\right|
\le2M\left(\sup_\lambda\norm{a_\lambda}_{\mu,2}\right)
\varepsilon^{1/2}.
\]
Letting \(\varepsilon\downarrow0\) proves the claim.
\end{proof}

\subsection{Proof of Theorem~\ref{thm:unified}}

\begin{proof}
Assume first that \(\mathrm{MSE}_\lambda\) remains bounded.
Set \(a_\lambda=g_\lambda-\bar U\) and \(b_\lambda=F_\lambda-\bar U\).
Since \(a_\lambda=(g_\lambda-F_\lambda)+b_\lambda\) and
\(\norm{b_\lambda}_{\mu,2}\le2M\), where
\(M=\sup_{\tau,p}\norm{U(\tau;p)}_X\), boundedness of the MSE gives
\(\sup_\lambda\norm{a_\lambda}_{\mu,2}<\infty\). Moreover,
\(\Lip(a_\lambda)\le L_\lambda+K_U=o(\lambda)\).
Lemmas~\ref{lem:rapid_phase_averaging} and
\ref{lem:rapid_phase_orthogonality}, together with the
Hilbert-space identity, give
\[
 \mathrm{MSE}_\lambda
 =\norm{a_\lambda}_{\mu,2}^2+\norm{b_\lambda}_{\mu,2}^2
       -2\langle a_\lambda,b_\lambda\rangle
 =\norm{g_\lambda-\bar U}_{\mu,2}^2+\sigma_{\mathrm{phase}}^2+o(1).
\]
This proves the decomposition. If the lower limit of the MSE is infinite,
the lower bound is immediate. Otherwise, choose a sequence tending to
infinity along which the MSE converges to its finite lower limit.
The decomposition applies on that bounded-error sequence and proves the bound.
\end{proof}

\section{FM2: Per-Sample OOD Error}
\label{app:pointwise}
\subsection{Setup and lemmas}
Throughout this appendix, assume
Assumption~\ref{ass:regularity} and
Assumption~\ref{ass:fm2_boundary}.
For $v:[0,W]\to X$, write
\[
\norm{v}_{W,2}^2
=\frac1W\int_0^W\norm{v(t)}_X^2\,dt.
\]
Thus $\norm{v}_{W,2}=\norm{v(W\,\cdot\,)}_{\mathcal H}$.
Define accumulated difference
\[
D_W(p)=
\norm{U(\,\cdot\,\omega(p);p)
      -U(\,\cdot\,\omega_b;p)}_{W,2}.
\]

Define largest predictor error on boundary in \(G\)
\[
\varepsilon_W
:=\sup_{r\in\Gamma\cap G}e_W(r).
\]
\begin{lemma}[Comparison with the training boundary]
\label{lem:fm2_boundary_comparison}
Under condition in Theorem~\ref{thm:fm2_pointwise}, there exist a neighborhood $\mathcal N_0\subset G$ of $r_0$
and a constant $C_\Gamma>0$ such that
\[
e_W(p)\ge
D_W(p)-\varepsilon_W
-C_\Gamma(L_W+K_U)d_\omega(p)
\qquad (p\in\mathcal N_0).
\]
\end{lemma}

\begin{proof}
Let $\nu=\nabla\omega(r_0)/\norm{\nabla\omega(r_0)}_2$, so that
$\omega$ increases in direction $\nu$ at $r_0$. For $p$ near $r_0$ with
$\omega(p)<\omega_b$, we move along $p+h\nu$, $h\ge0$, until $\omega$
reaches $\omega_b$; the case $\omega(p)>\omega_b$ is identical with
$\nu$ replaced by $-\nu$.

To control the rate, note that $\omega$ is $C^1$, so there are
$\rho,\kappa>0$ such that $B(r_0,\rho)\subset G$ and
$\partial_\nu\omega\ge\kappa$ on $B(r_0,\rho)$. Let $\mathcal N_0$ be the
set of $p$ with $\norm{p-r_0}_2<\rho/2$ and $d_\omega(p)<\kappa\rho/2$.
For $p\in\mathcal N_0$, since $\omega$ grows at rate at least $\kappa$
along the ray from $p$ in direction $\nu$, it reaches $\omega_b$ at a
point $r(p)\in\Gamma\cap G$ with
\[
\norm{p-r(p)}_2\le\kappa^{-1}d_\omega(p).
\]

Write $r=r(p)$. Since $\omega(r)=\omega_b$, we have
$F_W^{(2)}(r)=U(W\,\cdot\,\omega_b;r)$, so the triangle inequality gives
\[
\begin{aligned}
D_W(p)
&\le \norm{F_W^{(2)}(p)-g_W^{(2)}(p)}_{\mathcal H}
 +\norm{g_W^{(2)}(p)-g_W^{(2)}(r)}_{\mathcal H}
 +\norm{g_W^{(2)}(r)-F_W^{(2)}(r)}_{\mathcal H}\\
&\quad
 +\norm{U(W\,\cdot\,\omega_b;r)-U(W\,\cdot\,\omega_b;p)}_{\mathcal H}\\
&\le e_W(p)+L_W\norm{p-r}_2+\varepsilon_W+K_U\norm{p-r}_2.
\end{aligned}
\]
The claim follows from $\norm{p-r}_2\le\kappa^{-1}d_\omega(p)$, with
$C_\Gamma=\kappa^{-1}$.
\end{proof}

\begin{lemma}[Fast phase averaging]
\label{lem:fast_averaging}
Let $f:\Torus\times[0,1]\to[0,B]$ satisfy
$|f(\theta,s)-f(\theta,s')|\le L|s-s'|$, and set
$\bar f(s)=\int_{\Torus} f(\theta,s)\,d\theta$. Then for every $a>0$,
\[
\Bigl|\int_0^1 f(as,s)\,ds-\int_0^1\bar f(s)\,ds\Bigr|
\le\frac{L+B}{a}.
\]
\end{lemma}

\begin{proof}
Let $m=\lfloor a\rfloor$, $s_j=j/a$, $I_j=[s_j,s_{j+1}]$. As $s$ runs
over $I_j$, $as$ covers $\Torus$ exactly once, so
$\int_{I_j}f(as,s_j)\,ds=\bar f(s_j)/a$. Replacing $s_j$ by $s$ in the
second argument, on either side, costs at most
$L\int_{I_j}(s-s_j)\,ds=L/(2a^2)$ each, so the $m\le a$ full cells
contribute at most $L/a$. The remaining interval has length less than
$1/a$, and there $|f(as,s)-\bar f(s)|\le B$.
\end{proof}

\begin{lemma}[Phase separation over a long window]
\label{lem:fm2_phase_separation}
There exists a constant $M>0$ such that, for every
$p\in G$ and $W>0$, with
$a=W\omega_b$ and $\xi=W(\omega(p)-\omega_b)$,
\[
D_W(p)^2
\ge V(p)\min\{\xi^2,1\}
-\frac{4MK_\tau|\xi|}{a}
-\frac{4M^2}{a}.
\]
Moreover, $D_W(p)\le2M$.
\end{lemma}

\begin{proof}
Write $u_p=U(\cdot;p)-\bar U(p)$. By continuity and compactness,
$\norm{u_p(\tau)}_X\le M$ on $\Torus\times\mathcal P$, so $D_W(p)\le2M$.
Define
\[
A_p(v)=\int_0^1\norm{u_p(\tau+v)-u_p(\tau)}_X^2\,d\tau,
\qquad
f(\theta,s)=\norm{u_p(\theta+\xi s)-u_p(\theta)}_X^2 .
\]
Substituting $t=Ws$ and using
$Ws\omega_b=as$, $Ws\omega(p)=as+\xi s$ gives
\[
D_W(p)^2=\int_0^1\norm{u_p(as+\xi s)-u_p(as)}_X^2\,ds
=\int_0^1 f(as,s)\,ds.
\]The function $f$
is periodic with period one in $\theta$, satisfies $0\le f\le4M^2$, has
$\theta$ average $A_p(\xi s)$, and is $4MK_\tau|\xi|$ Lipschitz in $s$
by the difference of squares and the phase Lipschitz bound.
Lemma~\ref{lem:fast_averaging} gives
\[
D_W(p)^2\ge\int_0^1A_p(\xi s)\,ds-\frac{4MK_\tau|\xi|}{a}-\frac{4M^2}{a}.
\]

Denote by $\widehat u_{p,n}$ the $n$-th Fourier coefficient of $u_p$ in
$\tau$. Since $u_p$ is centered, $\widehat u_{p,0}=0$, and Parseval's
identity gives
\[
V(p)=\int_0^1
\norm{U(\tau;p)-\bar U(p)}_X^2\,d\tau=\sum_{n\ne0}\norm{\widehat u_{p,n}}^2.
\]

By the shift theorem, $u_p(\cdot+v)$ has coefficients
$e^{2\pi inv}\widehat u_{p,n}$, so $u_p(\cdot+v)-u_p$ has coefficients
$(e^{2\pi inv}-1)\widehat u_{p,n}$. Since
$|e^{2\pi inv}-1|^2=2(1-\cos2\pi nv)$, Parseval's identity applied to
$u_p(\cdot+v)-u_p$ gives
\[
A_p(v)=2\sum_{n\ne0}\norm{\widehat u_{p,n}}^2(1-\cos2\pi nv).
\]

Integrating term by term, which Tonelli permits since every term is
nonnegative,
\[
\int_0^1A_p(\xi s)\,ds
=\sum_{n\ne0}\norm{\widehat u_{p,n}}^2\,
2\Bigl(1-\tfrac{\sin(2\pi n\xi)}{2\pi n\xi}\Bigr)
\ge\sum_{n\ne0}\norm{\widehat u_{p,n}}^2\min\{(n\xi)^2,1\}
\ge V(p)\min\{\xi^2,1\}.
\]
The middle step uses
$2\bigl(1-\frac{\sin 2\pi x}{2\pi x}\bigr)\ge\min\{x^2,1\}$ for all
real $x$. When $|x|\le\frac12$, write
$1-\frac{\sin 2\pi x}{2\pi x}=\int_0^1(1-\cos2\pi xr)\,dr$ and apply
$1-\cos t\ge2t^2/\pi^2$ for $|t|\le\pi$, which gives the lower bound
$\frac{16}{3}x^2\ge x^2$. When $|x|\ge\frac12$,
$\frac{\sin 2\pi x}{2\pi x}\le\frac1\pi$, so the left side is at least
$2(1-\frac1\pi)\ge1$. The last step uses $|n|\ge1$. Plug into inequality of \(D_{W}(p)^2\), lemma is proved.
\end{proof}

\subsection{Derivation of per-sample bound}
\begin{proof}[Proof of Theorem~\ref{thm:fm2_pointwise}]
Let $\mathcal N_0$ be the neighborhood from
Lemma~\ref{lem:fm2_boundary_comparison}.
Since $V$ is continuous and $V(r_0)>0$, choose a smaller
neighborhood $\mathcal N\subset\mathcal N_0$ of $r_0$
and a constant $v_*>0$ such that
$V(p)\ge v_*$ on $\mathcal N$.

Fix $B>0$ and consider OOD points $p\in\mathcal N$
with $0<Wd_\omega(p)\le B$.
Lemma~\ref{lem:fm2_boundary_comparison} gives
\[
e_W(p)\ge D_W(p)-\delta_{W,B},
\qquad
\delta_{W,B}
:=\varepsilon_W+
C_\Gamma(L_W+K_U)\frac{B}{W}.
\]
Since $\varepsilon_W\to0$ and $L_W=o(W)$, we have
$\delta_{W,B}\to0$. Moreover,
\[
e_W(p)^2\ge D_W(p)^2-2D_W(p)\delta_{W,B}\ge D_W(p)^2-4M\delta_{W,B}.
\]
The first inequality follows by squaring $e_W(p)\ge D_W(p)-\delta_{W,B}$
when $D_W(p)\ge\delta_{W,B}$, and holds trivially otherwise because its
right side is negative. The second uses $D_W(p)\le2M$.

Set $a=W\omega_b$ and
$\xi=W(\omega(p)-\omega_b)$.
Since $|\xi|=Wd_\omega(p)\le B$,
Lemma~\ref{lem:fm2_phase_separation} implies
\[
D_W(p)^2
\ge V(p)\min\{(Wd_\omega(p))^2,1\}
-\frac{C_B}{W},
\qquad
C_B=\frac{4MK_\tau B+4M^2}{\omega_b}.
\]
Combining the two inequalities and dividing by
$V(p)\ge v_*$ yields
\[
\frac{e_W(p)^2}{V(p)}
\ge
\min\{(Wd_\omega(p))^2,1\}
-\frac{C_B/W+4M\delta_{W,B}}{v_*}.
\]
The final term tends to zero independently of $p$.
\end{proof}
\section{Validation across architectures}
\label{app:other-model}

\subsection{Experimental setup}

We repeat both sweeps on Goodwin and \textsc{gold95} with a 1D CNN
with 64 hidden channels and an MLP with three hidden layers of width 256.
Both models map the standardized parameter vector $p$ directly to the
trajectory on the output grid. We use the same data, splits, sweep values,
and evaluation protocol as in the FNO experiments. Each configuration is
trained for 100 epochs with five seeds. Further implementation details are
given in the training scripts in the accompanying code repository.

\subsection{Detailed results}

\textbf{FM1 shift sweep.} Figure~\ref{fig:other-c-sweep} compares the
FNO, CNN, and MLP. For all three architectures, the best and final
validation MSEs are below $\sigma^2_{\mathrm{phase}}$ at small shifts
and exceed it from $c=100$\,h for Goodwin and $c=1{,}000$\,h for
\textsc{gold95} onward.
At these larger shifts, the best validation MSE averaged over five seeds
is within about $2.2\%$ of the floor on both systems, while the final
validation MSE differs across architectures.
These results are consistent with Theorem~\ref{thm:unified}.

\begin{figure}[htbp]
    \centering
    \includegraphics[width=\linewidth]{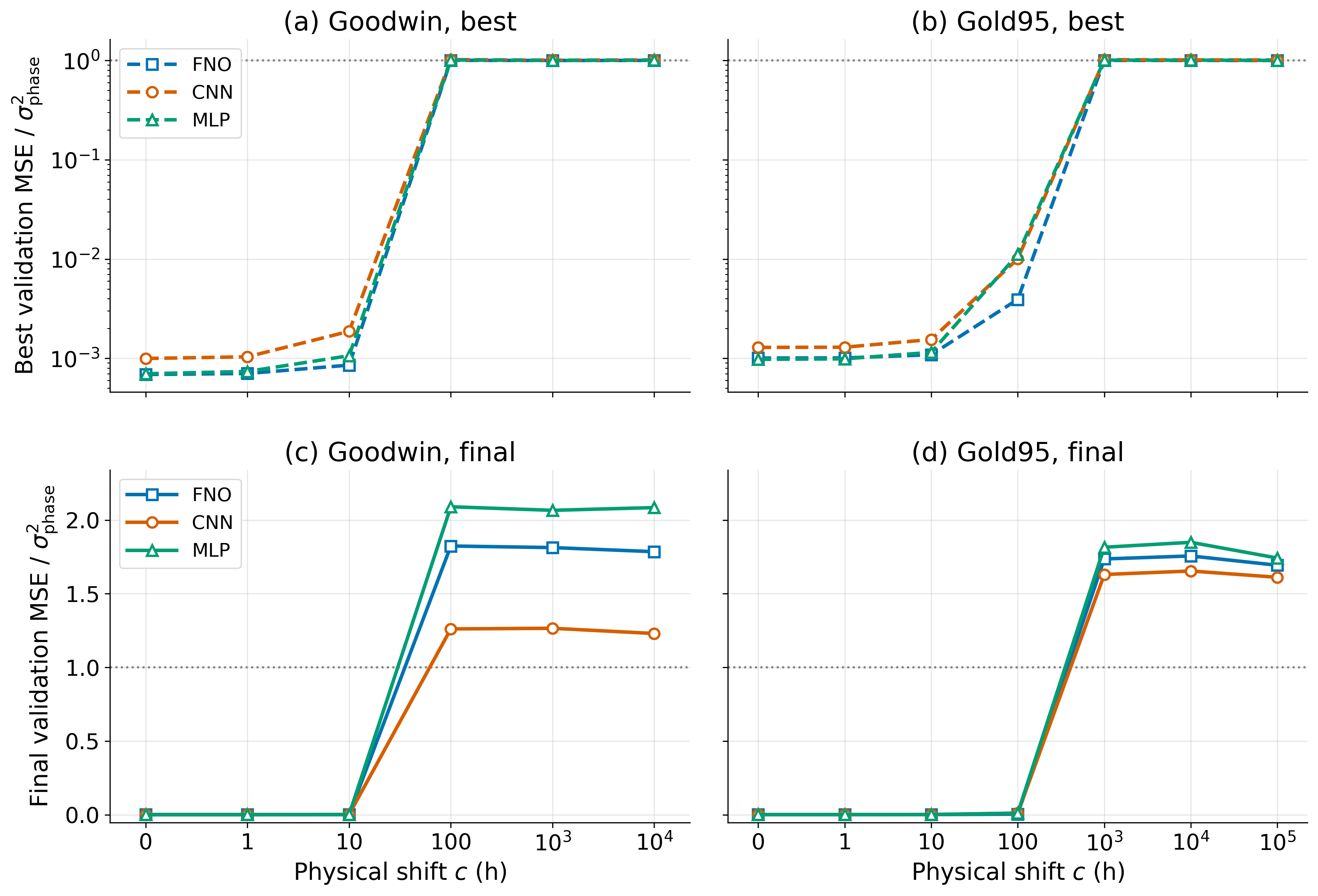}
    \caption{FM1 shift sweep comparing the FNO, CNN, and MLP on
    Goodwin (left) and \textsc{gold95} (right).
    The top row shows best validation MSE and the bottom row shows
    final validation MSE, both divided by $\sigma^2_{\mathrm{phase}}$.
    Markers show the mean over five seeds.
    The horizontal dotted line indicates the normalized error floor of $1$.}
    \label{fig:other-c-sweep}
\end{figure}

\textbf{FM2 population error.} Figure~\ref{fig:other-w-sweep} compares
the FNO, CNN, and MLP. For the CNN and MLP, both best and final
validation MSEs exceed $\sigma^2_{\mathrm{phase}}$ at every tested
window from $W=18.9$\,h for Goodwin and $W=97.7$\,h for
\textsc{gold95} onward. The FNO has lower errors at every matched
window. At the largest tested windows, the best validation MSEs
of all three architectures approach the floor, while their final
errors differ substantially. These results are consistent with
Theorem~\ref{thm:unified}.

\begin{figure}[htbp]
    \centering
    \includegraphics[width=\linewidth]{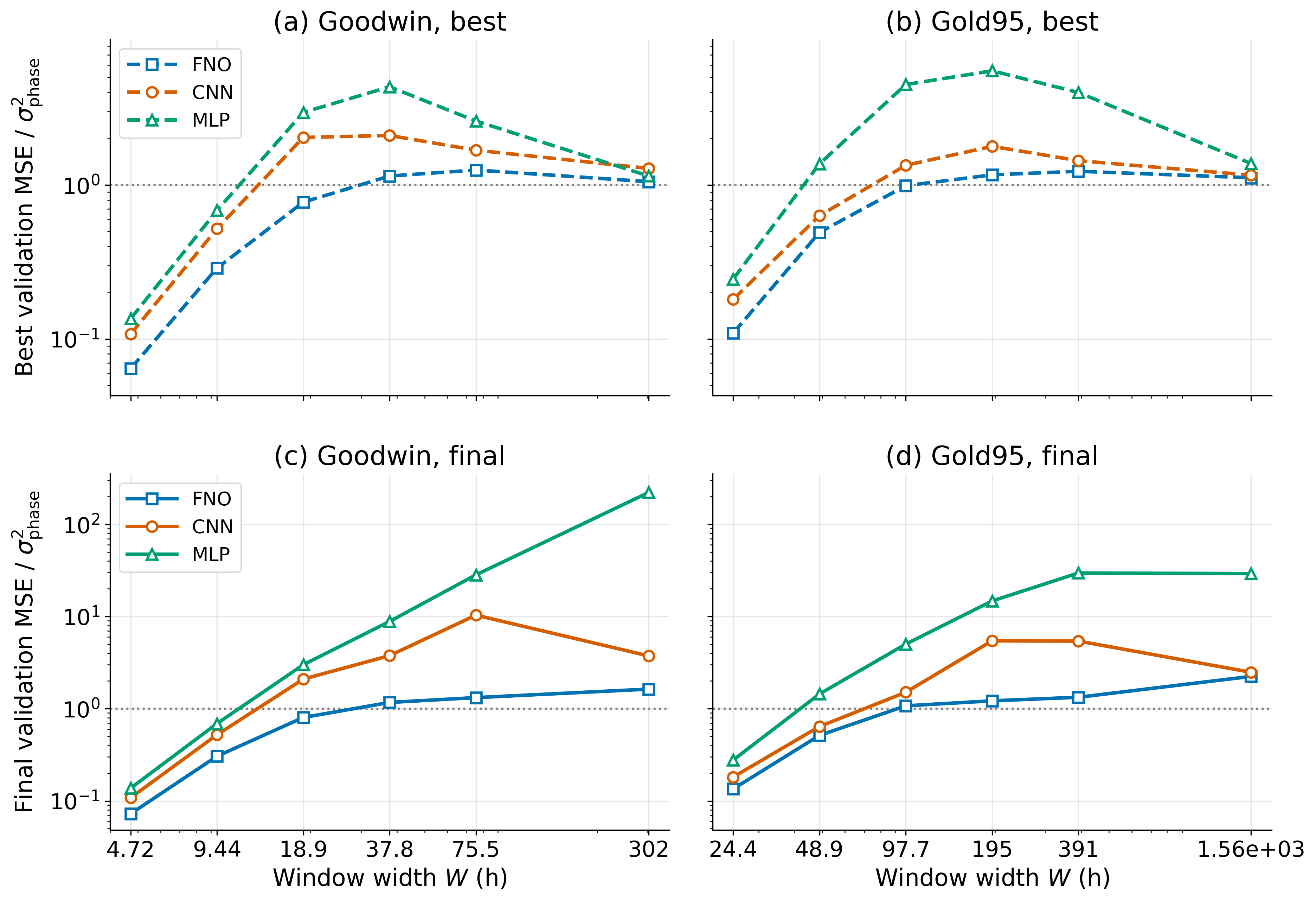}
    \caption{FM2 population error comparing the FNO, CNN, and MLP on
    Goodwin (left) and \textsc{gold95} (right).
    The top row shows best validation MSE and the bottom row shows
    final validation MSE, both divided by $\sigma^2_{\mathrm{phase}}$.
    Both axes use logarithmic scales. Markers show the mean over five
    seeds. The horizontal dotted line marks the normalized floor of $1$.}
    \label{fig:other-w-sweep}
\end{figure}

\textbf{FM2 per-sample OOD error.} Figures~\ref{fig:other-fm2-cnn}
and~\ref{fig:other-fm2-mlp} show that the mean normalized squared error
in each bin lies above $\min\{z^2,1\}$ at every tested $W$ for both
architectures. At the largest tested windows, $W=302.1$\,h for Goodwin
and $W=1563.4$\,h for \textsc{gold95}, the minimum in each bin also lies
above the curve. These results are consistent with the asymptotic bound
in Theorem~\ref{thm:fm2_pointwise}.

\begin{figure}[htbp]
    \centering
    \includegraphics[width=\linewidth]{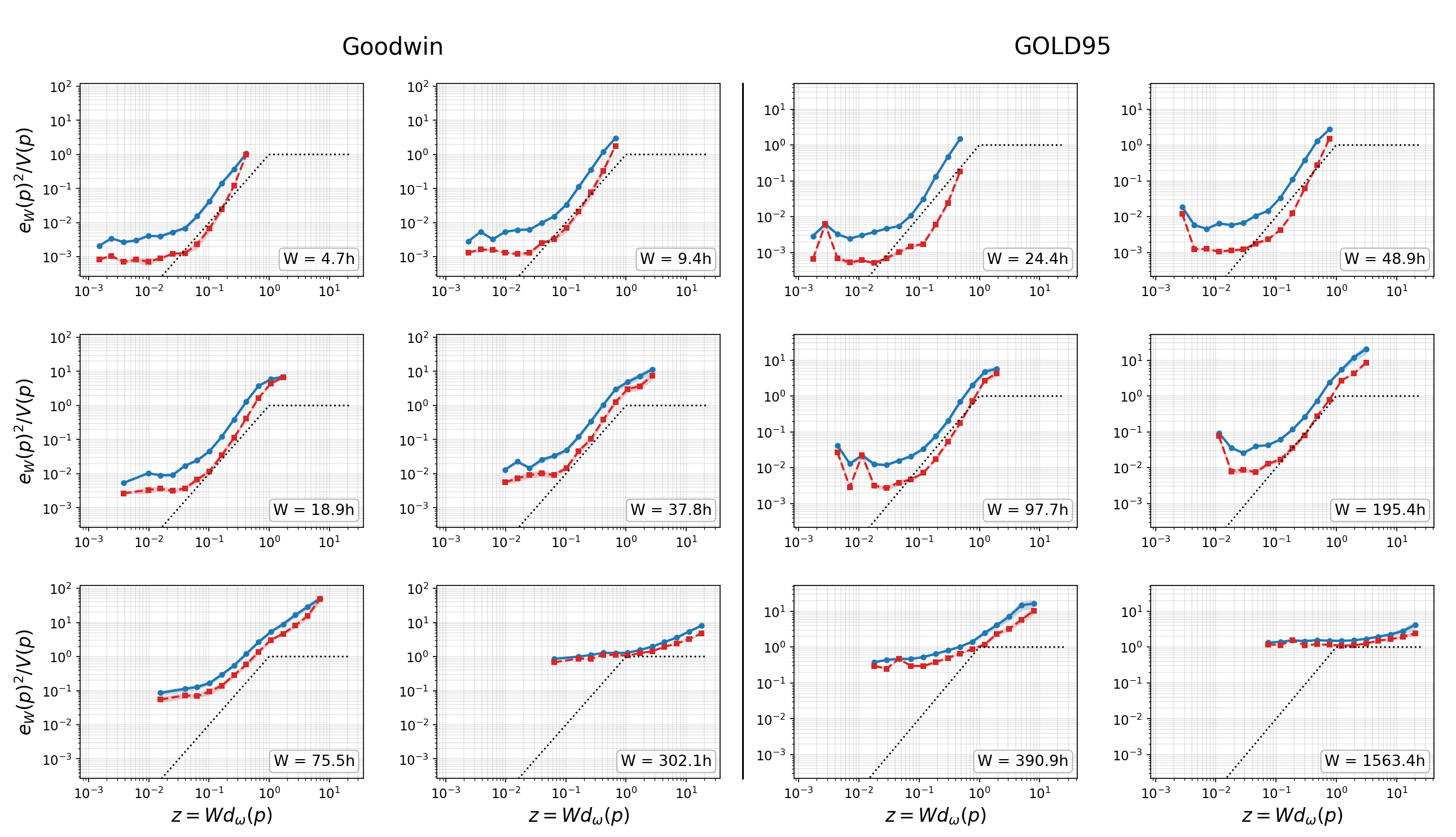}
    \caption{Per-sample OOD error for the CNN on Goodwin (left) and
    \textsc{gold95} (right). Blue curves show bin means and red curves
    show bin minima, each averaged over five seeds. Shading shows one
    standard deviation across seeds. The dotted line shows
    $\min\{z^2,1\}$.}
    \label{fig:other-fm2-cnn}
\end{figure}

\begin{figure}[htbp]
    \centering
    \includegraphics[width=\linewidth]{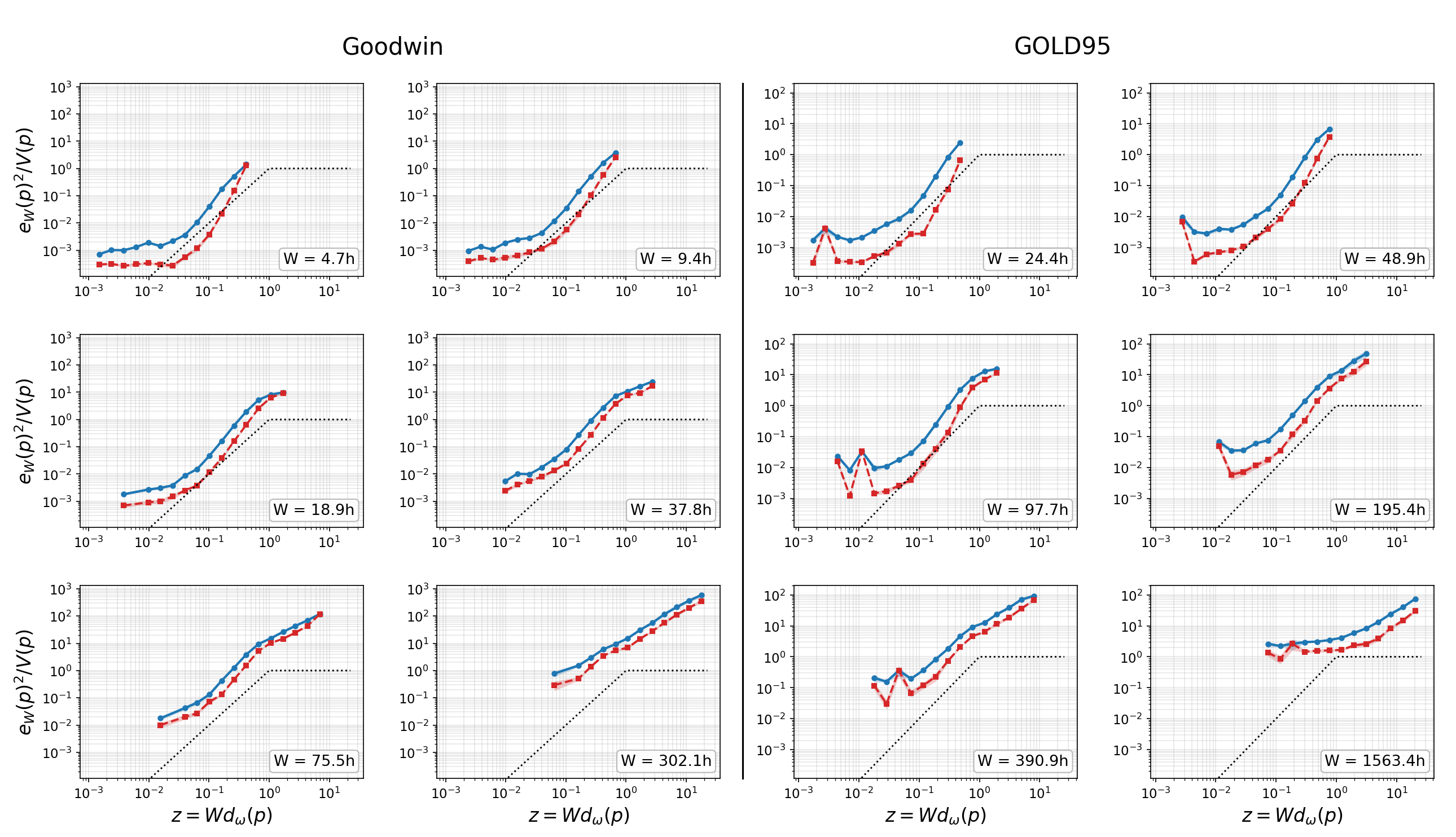}
    \caption{Per-sample OOD error for the MLP on Goodwin (left) and
    \textsc{gold95} (right). Blue curves show bin means and red curves
    show bin minima, each averaged over five seeds. Shading shows one
    standard deviation across seeds. The dotted line shows
    $\min\{z^2,1\}$.}
    \label{fig:other-fm2-mlp}
\end{figure}
\section{Empirical Lipschitz Estimation}
\label{app:lip_check}

\subsection{Lipschitz estimation}

For each final FNO checkpoint, we approximate the Jacobian of the
predicted trajectory with respect to the parameters by finite
differences at sampled validation points. Its largest singular value
gives a local Lipschitz estimate. We take the maximum of these
estimates as $\widehat L_\lambda$ for each seed and scale, using OOD
validation points in FM2. See the accompanying code for implementation
details.

\subsection{Lipschitz ratio sweeps}

Figures~\ref{fig:fm1_lip_ratio} and~\ref{fig:fm2_lip_ratio} show
$\widehat L_c/c$ and $\widehat L_W/W$, respectively. The lines show
means over five training seeds, and the shaded bands span the minimum
and maximum across those seeds. Both ratios decrease over the tested
scales.

\begin{figure}[htbp]
    \centering
    \includegraphics[width=\linewidth]{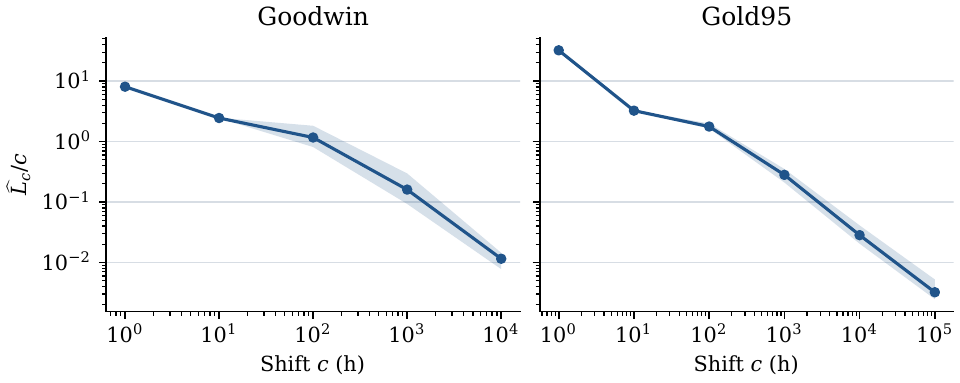}
    \caption{FM1 shift sweep. The ratio $\widehat L_c/c$ uses local
    Lipschitz estimates at validation points and is plotted for $c>0$.}
    \label{fig:fm1_lip_ratio}
\end{figure}

\begin{figure}[htbp]
    \centering
    \includegraphics[width=\linewidth]{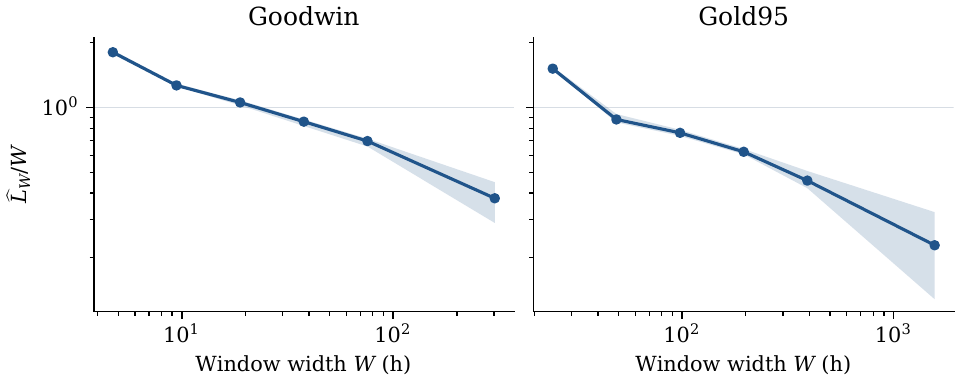}
    \caption{FM2 window sweep. The ratio $\widehat L_W/W$ uses local
    Lipschitz estimates at OOD validation points.}
    \label{fig:fm2_lip_ratio}
\end{figure}
\clearpage

\section{Proof of Theorem~\ref{thm:decoupled_continuity}}
\label{app:continuity_proofs}
Write $\Phi(t,x,p)$ for the flow of $\dot x=f(x,p)$.

\begin{proof}
Fix $p_0\in\mathcal P$ and a point $x_0\in\mathcal C(p_0)$, and choose
a transverse section $\Sigma$ through $x_0$. By
\citet[Lemma~12.7]{teschl2012ordinary}, the assumption $f\in C^2$ and
hyperbolicity give a local $C^2$ continuation $x_\Sigma(p)$ of the
corresponding fixed point of the Poincar\'e map. Moreover,
\citet[Lemma~6.9]{teschl2012ordinary} applied to the augmented system
$\dot x=f(x,p)$ and $\dot p=0$ shows that the local return time
$\tau_\Sigma(x,p)$ is $C^2$ jointly in $(x,p)$. Therefore,
\[
 T(p)=\tau_\Sigma(x_\Sigma(p),p)
\]
is $C^2$ locally. After shrinking the parameter neighborhood if
necessary, this first return is the minimal period. The unique-cycle
assumption identifies the continued orbit with $\mathcal C(p)$, so the
local definitions agree and $T$ is continuous on $\mathcal P$.

Define the local unanchored waveform by
\[
 \widetilde U(\tau,p)
 =\Phi\bigl(\tau T(p),x_\Sigma(p),p\bigr).
\]
Since the flow is $C^2$, $\widetilde U$ is jointly $C^2$ in $(\tau,p)$.
Let $A(\tau,p)=a_{\mathrm{obs}}(\widetilde U(\tau,p))$. Since
$a_{\mathrm{obs}}\in C^2$, the function $A$ is $C^2$. At $p_0$, its
unique maximizing phase is a nondegenerate critical point, so the
implicit function theorem gives a local $C^1$ continuation
$\tau_*(p)$. The strict gap between this maximum and the values outside
a small phase neighborhood ensures that $\tau_*(p)$ remains the unique
global maximizer for nearby $p$. Hence
\[
 U(\tau;p)=\widetilde U(\tau+\tau_*(p),p)
\]
is jointly continuous locally. Uniqueness of the anchor makes these
local definitions agree, so $U$ is jointly continuous on
$\mathbb T\times\mathcal P$.

It follows that $U_G$ is continuous. Since $\mathcal P$ is compact,
coordinatewise universal approximation gives uniform approximability
of $T$ and $U_G$
\citep{hornik1989multilayer,leshno1993multilayer}.
\end{proof}

\section{FM1 and FM2 Validation Experiments}
\label{app:new_experiments}

\subsection{FM1 shift sweep: configuration}
\label{app:fm1_config}

\paragraph{Data.} Each sample pairs a parameter $p$ with one period of its orbit on a normalized time grid. For a shift $c$, the target period starts $c$ hours after the peak. Inputs are the same for every $c$; only the targets change. We use $c \in \{0, 1, 10, 10^2, 10^3, 10^4\}$ hours for Goodwin and $c \in \{0, 1, 10, 10^2, 10^3, 10^4, 10^5\}$ hours for Gold95.

\paragraph{Split and scaling.} A single 80/20 train and validation split, fixed by split seed $0$, is shared by all $c$ and all seeds. This gives $4000/1000$ samples for Goodwin and $8000/2000$ for Gold95. Inputs are standardized, and targets are encoded by a power transform fitted on the training targets separately at each $c$. All losses are computed in this encoded space.

\paragraph{Phase variance.} For each validation orbit, we encode the base orbit on $512$ uniform phases and take the mean over phases and state variables of its squared deviation from the phase average. $\sigma^2_{\mathrm{phase}}$ is the average of this quantity over the validation set, with the same batch averaging as the validation loss. Because the encoding depends on $c$, $\sigma^2_{\mathrm{phase}}$ is computed separately at each $c$.

\paragraph{Model and training.} The FNO takes $(\tau, p)$ as input channels and outputs all state variables, with $16$ Fourier modes and $128$ hidden channels. We train with the MSE loss, Adam at initial learning rate $10^{-3}$, cosine annealing to $10^{-5}$ over $100$ epochs, and batch size $32$. Seeds $0$ to $4$ set the weight initialization and batch order (\texttt{torch.manual\_seed}). The validation loss is evaluated after every epoch: Final val MSE is its value after the last epoch, and Best val MSE is its minimum over all epochs.

\subsection{FM1 shift sweep: detailed results}
\label{app:fm1_sweep_results}

Detailed FM1 sweep results are summarized in Table~\ref{tab:fm1_summary}.
For each of the five seeds in rows marked $\dagger$, the best
validation checkpoint occurs in epochs 0--3. Afterward, training loss
continues to decrease while validation MSE rises. Best val MSE reports
the error at that checkpoint, and Final val MSE reports the error at
the last epoch. 

\begin{table}[h]
\centering
\footnotesize
\setlength{\tabcolsep}{4pt}
\caption{FM1 shift sweep, mean $\pm$ std over five seeds. Val MSE is divided by $\sigma^2_{\mathrm{phase}}$ at the same $c$; the $\sigma^2_{\mathrm{phase}}$ column gives its unnormalized value. Best epoch is the range over seeds, zero based.}
\label{tab:fm1_summary}
\begin{tabular}{llcccc}
\toprule
System & $c$ (h) & $\sigma^2_{\mathrm{phase}}$ & Best val MSE$/\sigma^2_{\mathrm{phase}}$ & Best epoch & Final val MSE$/\sigma^2_{\mathrm{phase}}$ \\
\midrule
Goodwin & $0$      & $0.5705$ & $6.870{\times}10^{-4} \pm 7.9{\times}10^{-6}$   & $94$--$99$ & $6.921{\times}10^{-4} \pm 1.30{\times}10^{-5}$ \\
        & $1$      & $0.5730$ & $7.034{\times}10^{-4} \pm 8.7{\times}10^{-6}$   & $96$--$99$ & $7.099{\times}10^{-4} \pm 1.09{\times}10^{-5}$ \\
        & $10$     & $0.5718$ & $8.539{\times}10^{-4} \pm 2.03{\times}10^{-5}$  & $96$--$99$ & $8.618{\times}10^{-4} \pm 2.57{\times}10^{-5}$ \\
        & $100$    & $0.5721$ & $1.0055 \pm 0.0018$ & $0$--$1$ & $1.8241 \pm 0.0293$ $\dagger$ \\
        & $10^3$   & $0.5721$ & $1.0013 \pm 0.0007$ & $1$--$3$ & $1.8134 \pm 0.0199$ $\dagger$ \\
        & $10^4$   & $0.5720$ & $1.0034 \pm 0.0005$ & $0$--$1$ & $1.7855 \pm 0.0338$ $\dagger$ \\
\midrule
Gold95  & $0$      & $0.7433$ & $1.0059{\times}10^{-3} \pm 1.90{\times}10^{-5}$ & $79$--$98$ & $1.0135{\times}10^{-3} \pm 1.99{\times}10^{-5}$ \\
        & $1$      & $0.7427$ & $1.0090{\times}10^{-3} \pm 1.23{\times}10^{-5}$ & $79$--$98$ & $1.0171{\times}10^{-3} \pm 1.43{\times}10^{-5}$ \\
        & $10$     & $0.7410$ & $1.0868{\times}10^{-3} \pm 1.60{\times}10^{-5}$ & $94$--$98$ & $1.0915{\times}10^{-3} \pm 1.56{\times}10^{-5}$ \\
        & $100$    & $0.7423$ & $3.918{\times}10^{-3} \pm 1.70{\times}10^{-4}$  & $95$--$98$ & $3.941{\times}10^{-3} \pm 1.87{\times}10^{-4}$ \\
        & $10^3$   & $0.7422$ & $1.0030 \pm 0.0004$ & $0$--$2$ & $1.7362 \pm 0.0185$ $\dagger$ \\
        & $10^4$   & $0.7423$ & $1.0030 \pm 0.0009$ & $0$--$2$ & $1.7566 \pm 0.0476$ $\dagger$ \\
        & $10^5$   & $0.7422$ & $1.0008 \pm 0.0009$ & $0$--$3$ & $1.6935 \pm 0.0326$ $\dagger$ \\
\bottomrule
\end{tabular}
\end{table}
\subsection{FM2 window sweep: configuration}
\label{app:FM2_config}

\paragraph{Network and optimization.} We use the neuralop 1D FNO implementation with a single fixed hidden width and neuralop's default depth and activation. Input channels are $1 + n_{\mathrm{params}}$ (a time encoding plus the parameter vector, broadcast along the grid) and output channels are $n_{\mathrm{vars}}$. Goodwin: $n_{\mathrm{vars}}=3$, $n_{\mathrm{params}}=8$. GOLD95: $n_{\mathrm{vars}}=5$, $n_{\mathrm{params}}=14$. Table~\ref{tab:fm2-hparams} lists the training hyperparameters, held fixed across the sweep.

\begin{table}[h]
\centering
\small
\caption{FM2 window sweep: training hyperparameters (identical for both systems).}
\label{tab:fm2-hparams}
\begin{tabular}{ll}
\toprule
Hidden channels & 128 \\
Optimizer & Adam \\
Learning rate & $10^{-3}$ \\
Weight decay & $0$ \\
LR schedule & cosine annealing, $\eta_{\min} = 10^{-5}$ \\
Epochs & 100 \\
Effective batch size & 32 (no gradient accumulation needed) \\
Validation frequency & every epoch \\
\bottomrule
\end{tabular}
\end{table}

\paragraph{Window sweep.} For each system we fix $W_0 = $ cutoff and sweep $W \in W_0 \cdot \{1,2,4,8,16,64\}$, holding $dt = W_0 / \text{base\_grid}$ constant so the grid scales as $\text{base\_grid} \cdot \text{factor}$ (base\_grid $=128$ for both systems). Goodwin uses a  cutoff of $4.72$h; GOLD95's cutoff is $24.428$h. The resulting $W$ values are:
\begin{itemize}
    \item Goodwin: $\{4.72, 9.44, 18.88, 37.76, 75.52, 302.08\}$h
    \item GOLD95: $\{24.43, 48.86, 97.71, 195.43, 390.85, 1563.40\}$h
\end{itemize}
Each $(W, \text{system})$ combination is trained with 5 seeds ($0$--$4$). Each seed ($0$--$4$) sets weight initialization (\texttt{torch.manual\_seed}) and, via an independent \texttt{torch.Generator}, minibatch shuffling.

\paragraph{Fixed spectral width.} For each system, one $n_{\mathrm{modes}}$ is chosen ($1728$ for Goodwin, $1984$ for GOLD95) at $2\times$ headroom over the harmonic demand at the largest $W$, and is held fixed across every $W$ in that system's sweep (neuralop's SpectralConv slices this down to the FFT length at smaller $W$, so a single fixed value is valid throughout). This gives $865$ kept rfft frequencies for Goodwin and $993$ for GOLD95.

\paragraph{Split and OOD definition.} Training/validation is a period-based split: samples with period $T <$ cutoff are training, $T \geq$ cutoff are validation (OOD in period, and hence in the window-mismatch variable $z$). Goodwin: $N=5000$ total, $2020$ train / $2980$ val. GOLD95: $N=10000$ total, $4822$ train / $5178$ val. The boundary and hardest-OOD reference sets used for the per-sample tracking use $n_{\mathrm{boundary}}=20$, $n_{\mathrm{hard}}=20$; the variance floor guard is $v_{\min\_\mathrm{frac}}=0.1$; windows are not endpoint-inclusive.

\subsection{FM2 window sweep: detailed results}
\label{app:FM2_results}

Tables~\ref{tab:fm2-goodwin} and~\ref{tab:fm2-gold95} report the
validation MSE normalized by $\sigma^2_{\mathrm{phase}}$ for each $W$.
Final and best-validation MSE are shown as mean $\pm$ one standard
deviation over five seeds.

\begin{table}[htbp]
    \centering
    \caption{Goodwin: $W$ sweep, five seeds per $W$.}
    \label{tab:fm2-goodwin}
    \begin{tabular}{rcc}
        \toprule
        $W$ (h) & Final MSE/$\sigma^2_{\mathrm{phase}}$ & Best MSE/$\sigma^2_{\mathrm{phase}}$ \\
        \midrule
        4.72   & 0.073 $\pm$ 0.009 & 0.064 $\pm$ 0.005 \\
        9.44   & 0.305 $\pm$ 0.007 & 0.289 $\pm$ 0.014 \\
        18.88  & 0.805 $\pm$ 0.027 & 0.774 $\pm$ 0.043 \\
        37.76  & 1.171 $\pm$ 0.058 & 1.144 $\pm$ 0.052 \\
        75.52  & 1.321 $\pm$ 0.143 & 1.251 $\pm$ 0.090 \\
        302.08 & 1.633 $\pm$ 0.141 & 1.051 $\pm$ 0.012 \\
        \bottomrule
    \end{tabular}
\end{table}

\begin{table}[htbp]
    \centering
    \caption{GOLD95: $W$ sweep, five seeds per $W$.}
    \label{tab:fm2-gold95}
    \begin{tabular}{rcc}
        \toprule
        $W$ (h) & Final MSE/$\sigma^2_{\mathrm{phase}}$ & Best MSE/$\sigma^2_{\mathrm{phase}}$ \\
        \midrule
        24.43   & 0.135 $\pm$ 0.015 & 0.109 $\pm$ 0.009 \\
        48.86   & 0.514 $\pm$ 0.050 & 0.492 $\pm$ 0.048 \\
        97.71   & 1.076 $\pm$ 0.042 & 0.990 $\pm$ 0.050 \\
        195.43  & 1.219 $\pm$ 0.043 & 1.164 $\pm$ 0.042 \\
        390.85  & 1.333 $\pm$ 0.069 & 1.228 $\pm$ 0.072 \\
        1563.40 & 2.244 $\pm$ 0.121 & 1.114 $\pm$ 0.042 \\
        \bottomrule
    \end{tabular}
\end{table}

\subsection{Per-sample analysis for FM2}
\label{app:fm2_per_sample}

For each system, we collect $z = W d_\omega(p)$ for every validation
parameter $p$ and every window $W$ in the sweep. Let $z_{\min}$ and
$z_{\max}$ be the smallest and largest of these values, which occur at the
shortest and longest window, respectively. We use the bin edges
$b_k = z_{\min}\, r^k$ for $k = 0,\ldots,21$ with
$r = (z_{\max}/z_{\min})^{1/21}$, and bin $k$ is $[b_k, b_{k+1})$. This gives
$z_{\min} = 1.18\times10^{-3}$, $z_{\max} = 22.5$ and $r = 1.60$ for
Goodwin, and $z_{\min} = 1.37\times10^{-3}$, $z_{\max} = 25.7$ and
$r = 1.60$ for GOLD95. The same edges are used for every $W$. 

The validation set is fixed, so every
seed is evaluated on the same parameters. We keep only validation parameters whose waveform variance satisfies
$V(p)\ge 0.1\,V(p_\partial)$, where $p_\partial$ is the training parameter
with the longest period, so that the ratio $e_W(p)^2/V(p)$ is not dominated
by waveforms with very small variance.

Table~\ref{tab:fm2-minbin} summarizes the per-sample results for each $W$.
``Bins below'' counts the nonempty bins in which the bin minimum, averaged
over seeds, lies below $\min\{z^2,1\}$ at the bin center; the bin mean never
does. The two RSD columns measure how much the bin mean and the bin minimum
vary across seeds. For each bin, we divide the standard deviation over the
5 seeds by the average over the 5 seeds, and report the median of this
ratio over bins.

\begin{table}[htbp]
    \centering
    \small
    \caption{Per-sample results over the $W$ sweep.}
    \label{tab:fm2-minbin}
    \begin{tabular}{lrccc}
        \toprule
        & & & RSD of & RSD of \\
        System & $W$ (h) & Bins below & bin mean (\%) & bin min (\%) \\
        \midrule
        Goodwin & 4.72    & 6 / 13 &  8.8 & 16.8 \\
                & 9.44    & 5 / 13 &  8.8 & 12.9 \\
                & 18.88   & 6 / 13 &  6.3 & 12.5 \\
                & 37.76   & 5 / 13 & 11.5 & 11.3 \\
                & 75.52   & 2 / 13 &  8.7 & 14.7 \\
                & 302.08  & 0 / 12 &  9.2 &  9.5 \\
        \midrule
        GOLD95  & 24.43   & 6 / 13 & 11.1 & 22.3 \\
                & 48.86   & 4 / 13 & 11.1 & 13.9 \\
                & 97.71   & 3 / 14 &  7.8 & 16.1 \\
                & 195.43  & 0 / 13 &  8.4 &  9.0 \\
                & 390.85  & 0 / 14 &  9.6 & 13.1 \\
                & 1563.40 & 0 / 13 &  2.4 &  7.5 \\
        \bottomrule
    \end{tabular}
\end{table}

\subsection{Applicability of Theorem~\ref{thm:fm2_pointwise}}
\label{app: verify_fm2_two_system}
Parameter notations are given in Section~\ref{app:odes}. Except one varied parameter, all other parameters are fixed at their baseline values. The period decreases monotonically in $\delta_1$ for Goodwin and increases monotonically in $v_d$ for \textsc{gold95}.

\begin{table}[t]
\centering
\caption{Single parameter period sweeps used to check Assumption~2. }
\label{tab:period_sweep}
\small
\begin{tabular}{l c c c}
\toprule
System & Varied & Range & Period (h)  \\
\midrule
Goodwin & $\delta_1$ & $[0.56, 1.04]$ & $[4.36, 5.53]$ \\
\addlinespace
\textsc{gold95} & $v_d$ & $[0.66, 1.26]$ & $[21.58, 26.03]$  \\
\bottomrule
\end{tabular}
\end{table}
\section{Snapshot-Conditioned Prediction Details}
\label{app:snapshot_details}
The supplied evaluation selects one full-state snapshot on each
post-transient validation orbit, predicts a period and waveform, and
matches the snapshot by periodic interpolation in encoded state space. It uses $N_\phi=100$
candidates per search level and $D_{\mathrm{search}}=3$ levels.

\begin{algorithm}[htbp]
\caption{Snapshot-based physical-time prediction}
\label{alg:snapshot}
\begin{algorithmic}[1]
\STATE Predict $\widehat T(p)$ and the waveform grid; construct interpolants $\widehat U$ in physical state units and $\widehat U_{\mathrm{enc}}$ in encoded state space, and encode $y_{\mathrm{obs}}$.
\STATE Set the initial search interval to one full cycle.
\FOR{$d=1,\ldots,D_{\mathrm{search}}$}
\STATE Evaluate $\norm{\widehat U_{\mathrm{enc}}(\phi;p)-y_{\mathrm{obs,enc}}}_2^2$ at $N_\phi$ equally spaced candidates.
\STATE Select a minimizing candidate $\widehat\phi_d$ and use its two neighboring candidates to define the next interval, treating phase modulo one.
\ENDFOR
\STATE Return $\widehat x(t)=\widehat U(\widehat\phi_{D_{\mathrm{search}}}+t/\widehat T(p);p)$ over the requested physical-time horizon.
\end{algorithmic}
\end{algorithm}

\paragraph{Evaluation metric.}
We compare the reconstructed and reference trajectories over one true
period, $t\in[0,T(p)]$ or five true periods, $t\in[0,5T(p)]$ , starting from the observed snapshot.
If the predicted period is shorter than the true period, we repeat the
predicted cycle to fill this interval; if it is longer, we truncate the
predicted trajectory at $T(p)$. Errors are evaluated at matching physical
times in the encoded state space of the training output scaler $E$,
as in Appendix~\ref{app:training_ablations}.

Specifically, let $t_j=T(p)j/(G-1)$, $j=0,\ldots,G-1$, and let
$x_{\mathrm{ref}}(t;p,y_{\mathrm{obs}})$ be the ODE solution initialized
at the snapshot. The per-sample reconstruction error is
$\mathrm{MSE}_{\mathrm{recon}}(p)=\frac{1}{nG}\sum_{j=0}^{G-1}
\|\widehat U_{\mathrm{enc}}((\widehat\phi+t_j/\widehat T(p))\bmod 1;p)
-E(x_{\mathrm{ref}}(t_j;p,y_{\mathrm{obs}}))\|_2^2$,
where $n$ is the number of state components.
We also report the squared relative period error
$e_T^2(p)=((\widehat T(p)-T(p))/T(p))^2$.
Both metrics are dimensionless.

\paragraph{Per-seed results.}
Tables~\ref{tab:snapshot_per_seed} and
\ref{tab:snapshot_per_seed_5cycles} report the per-seed IID and OOD
results over one and five true periods, respectively.
Within each setting, all seeds are evaluated on the same validation
samples. Snapshot times are drawn at random once for each validation
sample after burn-in and shared by all seeds.

\begin{table}[htbp]
\centering\small
\caption{Per-seed snapshot-based validation. Entries are mean / median /
95th percentile over validation samples. Reconstruction MSE is evaluated
in the encoded state space over one true period in physical time.
The sample mean of $e_T^2$ gives the period relative MSE reported in
Table~\ref{tab:snapshot_results}.}
\label{tab:snapshot_per_seed}
\begin{tabular}{llcc}
\toprule
System & Seed & Reconstruction MSE & Squared relative period error $e_T^2$\\
\midrule
\multicolumn{4}{c}{IID}\\
\midrule
Goodwin & 0 & $4.46{\times}10^{-4}$ / $3.60{\times}10^{-4}$ / $1.08{\times}10^{-3}$ & $2.55{\times}10^{-6}$ / $5.16{\times}10^{-7}$ / $6.54{\times}10^{-6}$\\
        & 1 & $3.78{\times}10^{-4}$ / $3.11{\times}10^{-4}$ / $9.08{\times}10^{-4}$ & $2.60{\times}10^{-6}$ / $4.55{\times}10^{-7}$ / $6.66{\times}10^{-6}$\\
        & 2 & $4.93{\times}10^{-4}$ / $3.91{\times}10^{-4}$ / $1.19{\times}10^{-3}$ & $2.56{\times}10^{-6}$ / $4.89{\times}10^{-7}$ / $6.06{\times}10^{-6}$\\
        & 3 & $4.05{\times}10^{-4}$ / $3.17{\times}10^{-4}$ / $1.01{\times}10^{-3}$ & $2.37{\times}10^{-6}$ / $5.22{\times}10^{-7}$ / $5.66{\times}10^{-6}$\\
        & 4 & $3.29{\times}10^{-4}$ / $2.60{\times}10^{-4}$ / $8.16{\times}10^{-4}$ & $2.99{\times}10^{-6}$ / $5.72{\times}10^{-7}$ / $7.75{\times}10^{-6}$\\
\midrule
Gold95  & 0 & $1.31{\times}10^{-3}$ / $6.25{\times}10^{-4}$ / $5.45{\times}10^{-3}$ & $3.86{\times}10^{-5}$ / $1.02{\times}10^{-5}$ / $1.42{\times}10^{-4}$\\
        & 1 & $1.28{\times}10^{-3}$ / $5.92{\times}10^{-4}$ / $5.48{\times}10^{-3}$ & $3.57{\times}10^{-5}$ / $8.45{\times}10^{-6}$ / $1.47{\times}10^{-4}$\\
        & 2 & $1.37{\times}10^{-3}$ / $6.25{\times}10^{-4}$ / $5.94{\times}10^{-3}$ & $3.54{\times}10^{-5}$ / $1.05{\times}10^{-5}$ / $1.38{\times}10^{-4}$\\
        & 3 & $1.35{\times}10^{-3}$ / $6.45{\times}10^{-4}$ / $5.74{\times}10^{-3}$ & $3.30{\times}10^{-5}$ / $9.08{\times}10^{-6}$ / $1.21{\times}10^{-4}$\\
        & 4 & $1.34{\times}10^{-3}$ / $6.26{\times}10^{-4}$ / $5.37{\times}10^{-3}$ & $4.03{\times}10^{-5}$ / $1.08{\times}10^{-5}$ / $1.67{\times}10^{-4}$\\
\midrule
\multicolumn{4}{c}{OOD}\\
\midrule
Goodwin & 0 & $2.83{\times}10^{-2}$ / $1.77{\times}10^{-2}$ / $8.82{\times}10^{-2}$ & $1.32{\times}10^{-3}$ / $9.05{\times}10^{-4}$ / $3.58{\times}10^{-3}$\\
        & 1 & $3.62{\times}10^{-2}$ / $2.28{\times}10^{-2}$ / $1.10{\times}10^{-1}$ & $1.81{\times}10^{-3}$ / $1.32{\times}10^{-3}$ / $4.99{\times}10^{-3}$\\
        & 2 & $3.46{\times}10^{-2}$ / $2.16{\times}10^{-2}$ / $1.06{\times}10^{-1}$ & $1.67{\times}10^{-3}$ / $1.17{\times}10^{-3}$ / $4.63{\times}10^{-3}$\\
        & 3 & $4.38{\times}10^{-2}$ / $2.60{\times}10^{-2}$ / $1.45{\times}10^{-1}$ & $2.44{\times}10^{-3}$ / $1.83{\times}10^{-3}$ / $6.18{\times}10^{-3}$\\
        & 4 & $3.74{\times}10^{-2}$ / $2.39{\times}10^{-2}$ / $1.14{\times}10^{-1}$ & $2.04{\times}10^{-3}$ / $1.53{\times}10^{-3}$ / $5.38{\times}10^{-3}$\\
\midrule
Gold95  & 0 & $3.86{\times}10^{-2}$ / $6.36{\times}10^{-3}$ / $1.78{\times}10^{-1}$ & $1.28{\times}10^{-3}$ / $2.49{\times}10^{-4}$ / $6.21{\times}10^{-3}$\\
        & 1 & $3.83{\times}10^{-2}$ / $6.20{\times}10^{-3}$ / $1.81{\times}10^{-1}$ & $1.26{\times}10^{-3}$ / $2.18{\times}10^{-4}$ / $6.15{\times}10^{-3}$\\
        & 2 & $4.28{\times}10^{-2}$ / $7.97{\times}10^{-3}$ / $1.87{\times}10^{-1}$ & $1.25{\times}10^{-3}$ / $2.31{\times}10^{-4}$ / $5.98{\times}10^{-3}$\\
        & 3 & $3.37{\times}10^{-2}$ / $6.19{\times}10^{-3}$ / $1.55{\times}10^{-1}$ & $1.12{\times}10^{-3}$ / $1.91{\times}10^{-4}$ / $5.38{\times}10^{-3}$\\
        & 4 & $3.85{\times}10^{-2}$ / $6.68{\times}10^{-3}$ / $1.73{\times}10^{-1}$ & $1.35{\times}10^{-3}$ / $2.62{\times}10^{-4}$ / $6.35{\times}10^{-3}$\\
\bottomrule
\end{tabular}
\end{table}
\begin{table}[htbp]
\centering\small
\caption{Per-seed physical-time reconstruction over five true periods.
Entries summarize reconstruction MSE in the encoded state space
over validation samples.}
\label{tab:snapshot_per_seed_5cycles}
\begin{tabular}{llccc}
\toprule
System & Seed & Mean & Median & 95th percentile\\
\midrule
\multicolumn{5}{c}{IID}\\
\midrule
Goodwin & 0 & $6.11{\times}10^{-4}$ & $4.41{\times}10^{-4}$ & $1.65{\times}10^{-3}$\\
        & 1 & $5.33{\times}10^{-4}$ & $3.79{\times}10^{-4}$ & $1.41{\times}10^{-3}$\\
        & 2 & $6.67{\times}10^{-4}$ & $4.55{\times}10^{-4}$ & $1.80{\times}10^{-3}$\\
        & 3 & $5.68{\times}10^{-4}$ & $3.91{\times}10^{-4}$ & $1.56{\times}10^{-3}$\\
        & 4 & $5.08{\times}10^{-4}$ & $3.48{\times}10^{-4}$ & $1.37{\times}10^{-3}$\\
\midrule
Gold95  & 0 & $6.75{\times}10^{-3}$ & $2.71{\times}10^{-3}$ & $2.32{\times}10^{-2}$\\
        & 1 & $6.10{\times}10^{-3}$ & $2.27{\times}10^{-3}$ & $2.04{\times}10^{-2}$\\
        & 2 & $7.05{\times}10^{-3}$ & $2.65{\times}10^{-3}$ & $2.67{\times}10^{-2}$\\
        & 3 & $6.50{\times}10^{-3}$ & $2.51{\times}10^{-3}$ & $2.29{\times}10^{-2}$\\
        & 4 & $7.49{\times}10^{-3}$ & $2.86{\times}10^{-3}$ & $2.65{\times}10^{-2}$\\
\midrule
\multicolumn{5}{c}{OOD}\\
\midrule
Goodwin & 0 & $4.54{\times}10^{-1}$ & $3.36{\times}10^{-1}$ & $1.24$\\
        & 1 & $5.96{\times}10^{-1}$ & $4.62{\times}10^{-1}$ & $1.53$\\
        & 2 & $5.43{\times}10^{-1}$ & $4.09{\times}10^{-1}$ & $1.44$\\
        & 3 & $7.64{\times}10^{-1}$ & $6.24{\times}10^{-1}$ & $1.82$\\
        & 4 & $6.53{\times}10^{-1}$ & $5.10{\times}10^{-1}$ & $1.59$\\
\midrule
Gold95  & 0 & $4.94{\times}10^{-1}$ & $7.05{\times}10^{-2}$ & $2.77$\\
        & 1 & $4.90{\times}10^{-1}$ & $6.24{\times}10^{-2}$ & $2.82$\\
        & 2 & $4.91{\times}10^{-1}$ & $6.11{\times}10^{-2}$ & $2.76$\\
        & 3 & $4.27{\times}10^{-1}$ & $5.47{\times}10^{-2}$ & $2.38$\\
        & 4 & $5.00{\times}10^{-1}$ & $7.44{\times}10^{-2}$ & $2.79$\\
\bottomrule
\end{tabular}
\end{table}
\section{FitzHugh--Nagumo PDE implementation details}
\label{app:fhn_details}

\paragraph{PDE and parameter domain.}
We consider the one-dimensional FitzHugh--Nagumo reaction--diffusion system
adapted from \citep{cebrian2024six},
\begin{align}
    u_t &= D_u u_{xx}-u^3+u-v, \\
    v_t &= D_v v_{xx}+\varepsilon(u-bv+a),
\end{align}
on a periodic spatial domain \(x\in[0,\ell]\). The fixed constants and
parameter ranges are summarized in Table~\ref{tab:fhn_pde_setup}. We generate
parameter vectors using a scrambled Sobol sequence and retain \(5000\) samples
whose post-transient solutions lie on the winding-one traveling-wave branch.

\begin{table}[htbp]
    \centering
    \caption{FitzHugh--Nagumo PDE and data-generation settings.}
    \label{tab:fhn_pde_setup}
    \begin{tabular}{lll}
        \toprule
        Quantity & Symbol & Value \\
        \midrule
        Activator diffusion & \(D_u\) & \(1\) \\
        Inhibitor diffusion & \(D_v\) & \(1\) \\
        Domain length & \(\ell\) & \(50\) \\
        Fixed reaction parameter & \(a\) & \(0.1\) \\
        Sampled parameter & \(b\) & \([0.35,0.65]\) \\
        Sampled parameter & \(\varepsilon\) & \([0.035,0.065]\) \\
        Spatial grid size & \(N_x\) & \(256\) \\
        PDE transient burn-in & \(T_{\mathrm{burn,PDE}}\) & \(600\) \\
        Fourier-phase probe window & \(T_{\mathrm{probe}}\) & \(200\) \\
        Retained samples & -- & \(5000\) \\
        \bottomrule
    \end{tabular}
\end{table}

\paragraph{Numerical simulation and initial condition.}
The spatial Laplacian is approximated by a centered second-order finite
difference on the periodic grid. The resulting ODE system is integrated using
an explicit fourth-order Runge--Kutta method, with a time step below the
explicit diffusion stability limit.

For each parameter sample, we first compute the limit cycle of the spatially
homogeneous FitzHugh--Nagumo kinetics. Different phases on this cycle are
assigned to successive spatial locations, with one full phase cycle
distributed around the periodic domain. The PDE is then integrated for
\(T_{\mathrm{burn,PDE}}=600\) before the post-transient traveling wave is
measured.

\paragraph{Period estimation and phase alignment.}
For a winding-one rigid traveling wave,
\[
    u_p(x,t)=Q_p(x+v_{\mathrm{tw}}(p)t),
\]
the first spatial Fourier coefficient satisfies
\[
    \widehat u_{p,1}(t)
    =
    \widehat u_{p,1}(0)
    \exp\left(\frac{2\pi i v_{\mathrm{tw}}(p)t}{\ell}\right).
\]
We estimate its angular velocity \(\nu(p)\) by fitting a linear function to
the unwrapped Fourier phase over a post-burn-in interval of length
\(T_{\mathrm{probe}}\), and define
\[
    T(p)=\frac{2\pi}{|\nu(p)|}.
\]
The canonical phase is fixed by
\[
    \arg\widehat u_{p,1}=0
    \pmod{2\pi}.
\]
The solution is advanced to the next occurrence of this phase before the
periodic field is recorded. This Fourier-phase condition plays the same role
as the peak-based anchor used for the ODE systems.

Starting from the anchor, one complete period is recorded and resampled to a
\(128\times128\) grid in normalized time and space. The stored canonical field
has shape \(128\times128\times2\), where the final dimension corresponds to
\(u\) and \(v\).

\paragraph{Construction of learning targets.}
All target representations for a given parameter sample are constructed from
the same phase-aligned one-period field. The ablations therefore change only
the representation while keeping the underlying periodic solution fixed.

For the normalized representation, the phase-aligned periodic field is
represented on
\[
    (\tau,x)\in[0,1)\times[0,\ell)
\]
using \(128\) points in normalized time and \(128\) points in space.

For the FM1 experiments, we impose a common observation burn-in
\[
    T_{\mathrm{burn,in}}=10000,
\]
which is distinct from the PDE transient burn-in
\(T_{\mathrm{burn,PDE}}=600\). It determines the observation phase
\[
    \phi_0(p)
    =
    \frac{T_{\mathrm{burn,in}}\bmod T(p)}{T(p)}.
\]
Experiment~2 records one normalized cycle from this phase, while
Experiment~3 uses the same phase to construct a fixed physical-time window.

For the iid fixed-window experiments, we use \(W=100\), with \(256\)
temporal points and \(128\) spatial points. This window contains approximately
\(1.79\)--\(2.70\) periods over the dataset. All shifted and fixed-window
targets are evaluated from the same phase-aligned periodic records using
Fourier interpolation.

\paragraph{Data splits.}
The iid experiments use \(4000\) training and \(1000\) validation samples.
The period-based OOD split is summarized in
Table~\ref{tab:fhn_ood_split}. Experiments~4 and~5 use exactly the same
training and validation samples.

\begin{table}[t]
    \centering
    \caption{Period-based OOD split for the FitzHugh--Nagumo experiments.}
    \label{tab:fhn_ood_split}
    \begin{tabular}{lccc}
        \toprule
        Set & Samples & Period range & Cycles in \(W_{\mathrm{OOD}}\) \\
        \midrule
        Training &
        \(2000\) &
        \(37.05\)--\(42.94\) &
        \(1.00\)--\(1.16\) \\
        Validation &
        \(500\) &
        \(52.43\)--\(55.85\) &
        \(0.77\)--\(0.82\) \\
        \bottomrule
    \end{tabular}
\end{table}

The remaining \(2500\) samples are unused. Experiment~4 predicts the
phase-aligned normalized field. Experiment~5 predicts on the common
physical-time window
\[
    W_{\mathrm{OOD}}
    =
    T_{\mathrm{train,max}}
    =
    42.94.
\]

\paragraph{Model and training configuration.}
The waveform predictor is a two-dimensional FNO. Its inputs are the temporal
coordinate, spatial coordinate, and parameters \((b,\varepsilon)\), and its
outputs are the two fields \(u\) and \(v\). For the normalized representation,
the temporal coordinate is \(\tau\in[0,1)\); for the fixed-window
representation, it is physical time. The period predictor is a separate MLP
mapping \(p=(b,\varepsilon)\) to \(T(p)\). Their configurations are summarized
in Table~\ref{tab:fhn_training}.

\begin{table}[t]
    \centering
    \caption{FitzHugh--Nagumo model and training configuration.}
    \label{tab:fhn_training}
    \begin{tabular}{lll}
        \toprule
        Setting & Waveform FNO & Period MLP \\
        \midrule
        Architecture &
        \(16\) modes/dim., \(64\) channels &
        depth \(2\), width \(128\) \\
        Activation &
        -- &
        GELU \\
        Batch size &
        \(16\) &
        \(32\) \\
        Epochs &
        \(50\) &
        \(200\) \\
        Optimizer &
        Adam &
        Adam \\
        Initial learning rate &
        \(10^{-3}\) &
        \(10^{-3}\) \\
        Schedule &
        cosine annealing &
        cosine annealing \\
        iid split &
        \(4000/1000\) &
        \(4000/1000\) \\
        OOD split &
        \(2000/500\) &
        \(2000/500\) \\
        \bottomrule
    \end{tabular}
\end{table}

The input parameters are standardized using statistics from the training
split. For the FNO, \(u\) and \(v\) are standardized separately using
training-set statistics. The MSE values in Table~\ref{tab:fhn_results} are
final-epoch losses in this standardized output space. For the period MLP,
both the parameters and the period are standardized during training.

Period prediction is evaluated in the original period scale using
\[
    \mathrm{NMSE}_T
    =
    \frac{
        \sum_i\bigl(\widehat T_i-T_i\bigr)^2
    }{
        \sum_i T_i^2
    }.
\]
The period predictor achieves validation NMSE
\(2.72\times10^{-9}\) on the iid split and
\(2.16\times10^{-3}\) on the OOD split.

\end{document}